\documentclass[letterpaper, 10 pt, conference]{ieeeconf}
\IEEEoverridecommandlockouts                              % This command is only needed if 
\let\labelindent\relax
\usepackage{amsfonts}       % An extended set of fonts for use in mathematics

\usepackage{amsthm}
\usepackage{mathtools}      % Loads amsmath
\usepackage{amssymb}        % provides AMS fonts and symbols ex. \mathbb{...}
\usepackage{filecontents}
\usepackage{graphicx}       % enhanced support for graphics
\usepackage{marginnote}     % Add notes
\usepackage{marvosym}       % mvat
\usepackage{overpic}        % Add text on figures
\usepackage{tabularx}
\usepackage{cite}
\usepackage{color}
\usepackage[linesnumbered,algoruled,boxed,lined,noend]{algorithm2e}
\usepackage[normalem]{ulem}
\usepackage{enumitem}
\usepackage[normalem]{ulem}

\usepackage{lipsum}

\usepackage{comment}
\usepackage[export]{adjustbox}% http://ctan.org/pkg/adjustbox
\usepackage{epstopdf}

\usepackage[font=small]{caption}% http://ctan.org/pkg/caption
\newif\ifdraft
\draftfalse
\ifdraft
\usepackage[paperheight=11in,paperwidth=9.5in,
			left=1.25in,right=1.25in,
			top=0.75in,bottom=0.75in,
			heightrounded,marginparwidth=1.2in,
			marginparsep=0.05in]{geometry}
\usepackage{xcolor}
\usepackage{xargs} 
\usepackage[textsize=footnotesize]{todonotes}
\newcommandx{\nt}[2][1=]{\todo[linecolor=red,
			backgroundcolor=red!10,bordercolor=red,#1]{#2}}
\newcommandx{\jy}[2][1=]{\todo[linecolor=green,
			backgroundcolor=green!10,bordercolor=green,#1]{JY:#2}}
\newcommandx{\kg}[2][1=]{\todo[linecolor=green,
			backgroundcolor=green!10,bordercolor=blue,#1]{KG:#2}}
\newcommandx{\sw}[2][1=]{\todo[linecolor=blue,
			backgroundcolor=orange!10,bordercolor=orange,#1]{SW:#2}}
\else
\newcommand{\nt}[1]{{}}
\newcommand{\kg}[1]{{}}
\newcommand{\sw}[1]{{}}
\newcommand{\jy}[1]{{}}
\fi

\newif\iftwocolumn
\twocolumntrue

\newtheorem{problem}{Problem}
\newtheorem{proposition}{Proposition}[section]
\newtheorem{lemma}{Lemma}[section]

\newtheorem{theorem}{Theorem}[section]
\theoremstyle{definition}

\theoremstyle{remark}

\SetKwProg{Fn}{Function}{}{}
\SetKwComment{Comment}{$\triangleright$\ }{}

\makeatletter
\def\subsubsection{\@startsection{subsubsection}% name
                                 {3}% level
                                 {\z@ \hspace*{1mm}}% indent (formerly \parindent)
                                 {0ex plus 0.1ex minus 0.1ex}% before skip
                                 {0ex}% after skip
                                 {\normalfont\normalsize\itshape}}% style
\makeatother

\newcommand{\D}{\mathcal{D}}

\def\ldg{G_{\mathcal A_1,\mathcal A_2}^{\,\ell}\xspace}
\def\ldgg{G^{\,\ell}\xspace}
\def\udg{G_{\mathcal A_1,\mathcal A_2}^{\,u}\xspace}

\def\toro{\texttt{TORO}\xspace}

\def\lrbm{\texttt{LRBM}\xspace} 
\def\urbm{\texttt{URBM}\xspace} 
\def\mrb{\textsc{MRB}\xspace}
\def\rb{\textsc{RB}\xspace}
\def\fvs{\textsc{MFVS}\xspace}

\def\DFSDP{\textsc{DFDP}\xspace}
\def\PQS{\textsc{PQS}\xspace}
\def\spp{\textsc{SepPlan}\xspace}
\def\vsp{\texttt{VSP}\xspace}
\def\minvs{\textsc{MinVS}\xspace}
\def\vertsep{\textsc{VS}\xspace}
\def\ilpmrb{\textsc{TB}_{\mrb}\xspace}
\def\ilptb{\textsc{TB}_{\mrb}\xspace}
\def\ilpfvs{\textsc{TB}_{\textsc{FVS}}\xspace}

\title{%\Large \bf
On Minimizing the Number of Running Buffers for Tabletop Rearrangement
}
\author{Kai Gao \qquad Si Wei Feng \qquad Jingjin Yu
\thanks{
The authors are with the Department of Computer Science, 
Rutgers, the State University of New Jersey, Piscataway, NJ, 
USA. E-Mails: \{{\tt kai.gao, siwei.feng, jingjin.yu}\}@rutgers.edu. 
This work is supported in part by NSF awards IIS-1734419, 
IIS-1845888, and CCF-1934924.
}%
}

\begin{document}

\maketitle
\thispagestyle{empty}
\pagestyle{empty}

%%%%%%%%%%%%%%%%%%%%%%%%%%%%%%%%%%%%%%%%%%%%%%%%%%%%%%%%%%%%%%%%
% Latex editing related instructions for draft mode
%%%%%%%%%%%%%%%%%%%%%%%%%%%%%%%%%%%%%%%%%%%%%%%%%%%%%%%%%%%%%%%%
\ifdraft
\begin{picture}(0,0)%
\put(-12,105){
\framebox(505,40){\parbox{\dimexpr2\linewidth+\fboxsep-\fboxrule}{
\textcolor{blue}{
The file is formatted to look identical to the final compiled IEEE 
conference PDF, with additional margins added for making margin 
notes. Use $\backslash$todo$\{$...$\}$ for general side comments
and $\backslash$jy$\{$...$\}$ for JJ's comments. Set 
$\backslash$drafttrue to $\backslash$draftfalse to remove the 
formatting. 
}}}}
\end{picture}
\vspace*{-5mm}
\fi

%%%%%%%%%%%%%%%%%%%%%%%%%%%%%%%%%%%%%%%%%%%%%%%%%%%%%%%%%%%%%%%%
% Main text
%%%%%%%%%%%%%%%%%%%%%%%%%%%%%%%%%%%%%%%%%%%%%%%%%%%%%%%%%%%%%%%%
\begin{abstract}
For tabletop rearrangement problems with overhand grasps, storage space 
outside the tabletop workspace, or buffers, can temporarily hold objects 
which greatly facilitates the resolution of a given rearrangement task. 
This brings forth the natural question of how many running buffers are 
required so that certain classes of tabletop rearrangement problems are 
feasible.
In this work, we examine the problem for both the labeled (where each
object has a specific goal pose) and the unlabeled (where goal poses
of objects are interchangeable) settings. 
On the structural side, we observe that finding the minimum number of 
running buffers (\mrb) can be carried out on a dependency graph abstracted 
from a problem instance, and show that computing \mrb on dependency graphs 
is NP-hard. 
We then prove that under both labeled and unlabeled settings, even for 
uniform cylindrical objects, the number of required running buffers may 
grow unbounded as the number of objects to be rearranged increases; 
we further show that the bound for the unlabeled case is tight. 
On the algorithmic side, we develop highly effective algorithms for finding 
\mrb for both labeled and unlabeled tabletop rearrangement problems, 
scalable to over a hundred objects under very high object density. 
Employing these algorithms, empirical evaluations show that random 
labeled and unlabeled instances, which more closely mimics real-world 
setups, have much smaller \mrb{s}. 
\end{abstract}

\section{Introduction}\label{sec:intro}
In nearly all aspects of our everyday lives, be it work related, at home, or for 
play, objects are to be grasped and rearranged, e.g., tidying up a messy desk, 
cleaning the table after dinner, or solving a jigsaw puzzle. Similarly, many 
industrial and logistics applications require repetitive rearrangements of many 
objects, e.g., the sorting and packaging of products on conveyors with robots, 
%\jy{A figure with some examples?}
and doing so efficiently is of critical importance to boost the competitiveness 
of the stakeholders. However, even without the challenge of grasping, deciding 
the sequence of objects for optimizing a 
rearrangement task is non-trivial. To that end, Han et al. \cite{han2018complexity} 
examined the problem of \emph{tabletop object rearrangement with overhand grasps} 
(\toro), where objects may be picked up, moved around, and then placed at poses that 
are not in collision with other objects. An object that is picked up but cannot be 
directly placed at its goal is temporarily stored at a \emph{buffer} location.
For example, for the setup given in Fig.~\ref{fig:toro}, using a single manipulator,
either the Coke can or the Pepsi can must be moved to a buffer before the task can 
be completed. They show that computing a pick-n-place sequence that minimizes 
the use of the total number of buffers is NP-hard and provide fast methods for 
computing that solution for problems with a couple dozen objects.
\begin{figure}[ht]
    \centering
\vspace{2mm}
\begin{overpic}
[width=1\columnwidth]{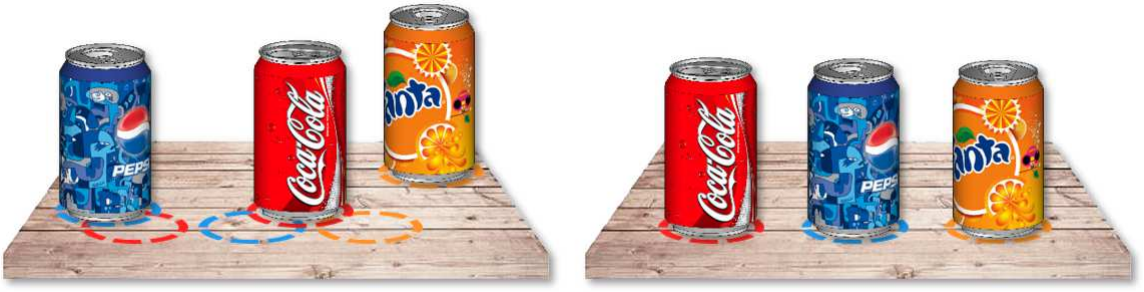}
\end{overpic}
%\vspace{1mm}
\caption{A \toro instance where the three soda cans are to be rearranged 
from the left configuration to the right configuration.}
\label{fig:toro}
\vspace{-1.5mm}
\end{figure}

In this study, we examine a more practical rearrangement objective which minimizes 
the number of \emph{running buffers} (\rb) needed for solving a \toro instance. 
We seek plans that minimize the maximum number of objects stored at buffers at any 
given moment. We denote this quantity as \mrb (minimum running buffer). The objective
is important because if the \mrb required for solving a \toro instance exceeds the available 
buffer storage, which is always limited in practice, then the instance is infeasible.
Therefore, the structural results and the algorithms that we present may be used not 
only for computing feasible and high-quality rearrangement plans, but they are also 
invaluable as a verification tool, e.g., to verify that a certain rearrangement setup 
will be able to solve most tasks for which it is designed to tackle.

The main technical contribution of this work is three-fold. First, we show that 
computing \mrb on arbitrary \emph{dependency graphs}, which encode the combinatorial 
information of \toro instances, is NP-hard. Second, we establish that for
an $n$-object \toro instance, \mrb can be lower bounded by $\Omega(\sqrt{n})$ for uniform
cylinders, even when all objects are \emph{unlabeled}. This implies that the 
same is true for the \emph{labeled} setting. Then, we provide a matching algorithmic upper
bound $O(\sqrt{n})$ for the unlabeled setting. 
%\sw{for all instances}
%\jy{I think this is not necessary to state, since that is what upper bound means}
Last but not least, we develop multiple highly effective and optimal algorithms for 
computing rearrangement plans with \mrb for \toro. 
In particular, we present a dynamic programming method for the labeled setting, 
a priority queue-based algorithm for the unlabeled setting, 
and a much more efficient \emph{depth-first-search dynamic programming} routine that readily scales to instances with over a hundred objects for both settings. 
Furthermore, we provide methods for computing plans with the minimum 
number of \emph{total buffers} subject to the \mrb constraints. 

\textbf{Related work}. As a high utility capability, manipulation of objects in 
a bounded workspace has been extensively studied, with works devoted to 
perception/scene understanding \cite{saxena2008robotic,gualtieri2016high,
mitash2017self,xiang2017posecnn}, task/rearrangement planning 
\cite{ben1998practical,stilman2005navigation,
treleaven2013asymptotically,havur2014geometric,haustein2015kinodynamic,
krontiris2015dealing,king2016rearrangement,
han2018complexity,huang2019large,lee2019efficient,pan2020decision}, 
manipulation \cite{taylor1987sensor,goldberg1993orienting,
lynch1999dynamic,dogar2011framework,bohg2013data,dafle2014extrinsic,
boularias2015learning,chavan2015prehensile}, as well as integrated 
holistic approaches \cite{kaelbling2011hierarchical,levine2016end,mahler2017dex,
zeng2018robotic,wells2019learning}.
As object rearrangement problems often embed within them multi-robot motion planning 
problems, rearrangement inherits the PSPACE-hard complexity \cite{hopcroft1984complexity}. 
These problems remain  NP-hard even without complex geometric constraints 
\cite{wilfong1991motion}. Considering rearrangement plan quality, e.g, minimizing the 
number of pick-n-places or the end-effector travel, is also computationally 
intractable \cite{han2018complexity}. 

For rearrangement tasks using mainly prehensile actions, the algorithmic 
studies of Navigation Among Movable Obstacles \cite{stilman2005navigation,
stilman2007manipulation} result in backtracking search methods that can 
effectively deal with monotone and other instances with ``nice'' 
properties.
Via carefully calling monotone solvers, difficult non-monotone cases can 
be solved as well \cite{krontiris2015dealing}.
Han et al. \cite{han2018complexity} relates tabletop rearrangement problems 
to the Traveling Salesperson Problem \cite{papadimitriou1977euclidean}  
and the Feedback Vertex Set problem \cite{karp1972reducibility},
both of which are NP-hard. Nevertheless, integer linear programming 
models are shown to quickly compute high quality  solutions for 
practical sized (e.g., 1-2 dozen of objects) problems. 
Focusing mainly on the unlabeled setting, bounds on the number of 
pick-n-places are provided under different assumptions on disk 
objects in \cite{bereg2006lifting}.
In \cite{lee2019efficient}, a complete algorithm is developed that reasons 
about object retrieval, rearranging other objects as needed, with later 
work \cite{nam2019planning} considering plan optimality and sensor 
occlusion. 
While the objectives in most problems focus on the number 
of motions, Halperin et al. \cite{halperin2020space} examined
space-aware reconfiguration in which disc objects move along straight 
lines in the workspace, seeking to minimize space needed to 
carry out a rearrangement task.

Non-prehensile rearrangement has also been extensively studied, with 
singulation as an early focus \cite{chang2012interactive,
laskey2016robot,eitel2020learning}. Iterative search was employed in 
\cite{huang2019large} for accomplishing a multitude of rearrangement 
tasks spanning singulating, separation, and sorting of identically 
shaped cubes.
Song et al. \cite{song2019multi} combines Monte Carlo Tree Search
with a deep policy network for separating many objects into coherent clusters
within a bounded workspace, supporting non-convex objects. 
More recently, a bi-level planner is proposed \cite{pan2020decision}, 
engaging both (non-prehensile) pushing and (prehensile) overhand grasping 
for sorting a large number of objects. 

On the structural side, a central object that we study is the \emph{dependency graph}
\cite{van2009centralized} structure. We observe that, in the labeled setting, through 
the dependency graph, the running buffer problems naturally connect to \emph{graph 
layout} problems \cite{diaz2002survey, garey1979, papadimitriou1976np, garey1974some, 
gavril1977, bodlaender1995approximating},
% \sw{\cite{garey1979, papadimitriou1976np, garey1974some, gavril1977, bodlaender1995approximating} is also related.} 
where an optimal linear ordering of graph vertices is sought. Graph layout 
problems find a vast number of important applications including VLSI design, scheduling 
\cite{shin2011minimizing}, and so on. For the unlabeled 
setting, the dependency graph becomes a planar one for uniform objects with a square 
or round base. Rearrangement can be tackled through partitioning of the dependency 
graph using a \emph{vertex separator} 
\cite{lipton1979separator,gilbert1984separator,alon1990separator,elsner1997graph}. For a survey on 
these topics, see~\cite{diaz2002survey}. 
% \sw{I feel survey on "Graph Partition" is a bit narrow, maybe survey [40] is better here?}

\textbf{Paper organization}. The rest of the paper is organized as follows. We 
introduce the \mrb focused rearrangement problems and discuss the associated 
dependency graphs in Sec.~\ref{sec:problem}. Then, in Sec.~\ref{sec:structure}, we 
establish some basic structural properties of the optimality structure of the 
problems, and show that minimizing running buffer size on dependency graphs is 
computationally intractable. We proceed to establish the lower and upper bounds 
on \mrb in Sec.~\ref{sec:bounds} and describe our proposed algorithmic solutions 
in Sec.~\ref{sec:algorithms}. Evaluation follows in Sec.~\ref{sec:evaluation}. 
We conclude with Sec.~\ref{sec:conclusion}.

\section{Preliminaries}\label{sec:problem}
We describe two practical (labeled and unlabeled) formulations of the tabletop 
object rearrangement problems using external buffers, and discuss the important 
\emph{dependency graph} structure for both settings.

\vspace{-1mm}
\subsection{Labeled Tabletop Rearrangement with External Buffers}
\vspace{-1mm}
Consider a bounded workspace $\mathcal W \subset \mathbb R^2$ with a set of 
$n$ objects $\mathcal O = \{o_1, \ldots, o_n\}$ placed inside it. All objects 
are assumed to be \emph{generalized cylinders} with the same height. A 
\emph{feasible arrangement} of these objects is a set of poses $\mathcal A 
=\{x_1,\ldots, x_n\}, x_i \in SE(2)$ in which no two objects collide. 
Let $\mathcal A_1 = \{x_1^s, \ldots, x_n^s\}$ and $\mathcal A_2 = 
\{x_1^g, \ldots, x_n^g\}$ be two feasible arrangements, a tabletop object 
rearrangement problem \cite{han2018complexity} seeks a plan using 
\emph{pick-n-place} operations that move the objects from $\mathcal A_1$ to 
$\mathcal A_2$ (see Fig.~\ref{fig:ex-prob}(a) for an 
example with 7 uniform cylinders). In each pick-n-place operation, an object 
is grasped by a robot arm, lifted above all other objects, transferred to and 
lowered at a new pose $p \in SE(2)$ where the object will not be in collision 
with other objects, and then released. A pick-n-place operation can be 
formally represented as a 3-tuple $a = (i, x', x'')$, denoting that object 
$o_i$ is moved from pose $x'$ to pose $x''$. A full rearrangement plan $P = (a_1, a_2, 
\ldots)$ is then an ordered sequence of pick-n-place operations. 

\begin{figure}[h]
    \centering
\begin{overpic}
[width=0.45\textwidth]{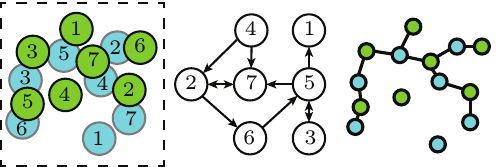}
\begin{comment}
\put(56, 10){{\small $1$}}
\put(66, 66){{\small $2$}}
\put(12, 47){{\small $3$}}
\put(57, 43){{\small $4$}}
\put(35, 62){{\small $5$}}
\put(10, 17){{\small $6$}}
\put(74, 23){{\small $7$}}

\put(43, 79){{\small $1$}}
\put(75, 40){{\small $2$}}
\put(16, 63){{\small $3$}}
\put(35, 37){{\small $4$}}
\put(13, 30){{\small $5$}}
\put(80, 66){{\small $6$}}
\put(53, 57){{\small $7$}}
\end{comment}
\put(13, -4){{\footnotesize (a)}}
\put(48, -4){{\footnotesize (b)}}
\put(80, -4){{\footnotesize (c)}}
\end{overpic}
\begin{comment}
\begin{overpic}
[width=0.15\textwidth]{ldg-example-eps-converted-to.pdf}

\put(76, 76){{\small $1$}}
\put(11, 46){{\small $2$}}
\put(76, 16){{\small $3$}}
\put(44, 76){{\small $4$}}
\put(76, 46){{\small $5$}}
\put(44, 16){{\small $6$}}
\put(44, 46){{\small $7$}}
\put(45, -12){{\footnotesize (b)}}
\end{overpic}
\begin{overpic}
[width=0.15\textwidth]{udg-example-eps-converted-to.pdf}
\put(45, -12){{\footnotesize (c)}}
\end{overpic}
\end{comment}
\vspace{4.5mm}
    %\begin{tabular}{ccc}
    %\includegraphics[width=0.15\textwidth]{../figures/example.png}
    %\includegraphics[width=0.15\textwidth]{../figures/labeled_DG_example.png}
    %\includegraphics[width=0.15\textwidth]{../figures/unlabeled_DG_example.png}
    %\end{tabular}
    \caption{A 7-object labeled instance with uniform cylinders; we 
    will use this instance as a running example. (a) The 
    green discs (as projections of cylinders) represent the start arrangement
    $\mathcal{A}_1$ and the cyan discs represent the goal arrangement $\mathcal{A}_2$. 
    (b) The corresponding labeled dependency graph. (c) The corresponding 
    unlabeled dependency graph, which is bipartite and planar.}
    \label{fig:ex-prob}
    \vspace{0mm}
\end{figure}
%\jy{ todo: add labels.}

Depending on $\mathcal A_1$ and $\mathcal A_2$, it may not always be 
possible to directly transfer an object $o_i$ from $x_i^s$ to $x_i^g$
in a single pick-n-place operation, because $x_i^g$ may be occupied 
by other objects. This creates \emph{dependencies} between objects. 
If object $o_i$ at pose $x^g_i$ intersects object $o_j$ at 
pose $x^s_j$, we say $o_i$ \emph{depends} on $o_j$. This suggests that 
object $o_j$ must be moved first before $o_i$ can be placed at its goal
pose $x^g_i$. 

It is possible to have circular dependencies, e.g., between objects $3$
and $5$ in Fig.~\ref{fig:ex-prob}(a). In such cases, some object(s) 
must be temporarily moved to an intermediate pose to solve the rearrangement
problem. Similar to \cite{han2018complexity}, we assume that \emph{external 
buffers} outside of the workspace are used for assuming intermediate poses, 
which avoids time-consuming geometric computations if the intermediate poses 
are to be placed within $\mathcal W$. During the execution of a 
rearrangement plan, there can be multiple objects that are stored at buffer 
locations. We call the buffers currently being used as \emph{running 
buffers} (\rb). %
With the introduction of buffers, there are three types of pick-n-place 
operations: 1) pick an object at its start pose and place at a buffer,
2) pick an object at its start pose and place at its goal pose, and 3) 
pick an object from buffer and place at its goal pose. Notice that buffer 
poses are not important. 
Naturally, it is desirable to be able to solve a rearrangement problem
with the least number of running buffers, yielding the running buffer 
minimization problem. 

\vspace{-1mm}
\begin{problem}[Labeled Running Buffer Minimization (\lrbm)]\label{p:1} Given 
feasible arrangements $\mathcal A_1$ and $\mathcal A_2$, find a rearrangement 
plan $P$ that minimizes the maximum number of running buffers used at any 
given time. 
\end{problem}
\vspace{-1mm}

In an \lrbm instance, the set of all dependencies induced by $\mathcal 
A_1$ and $\mathcal A_2$ can be represented using a directed graph $\ldg = 
(V, A)$, where each $v_i \in V$ corresponds to object $o_i$ and there is an 
arc $v_i \to v_j$ for $1 \le i, j \le n, i \ne j$ if object $o_i$ 
depends on object $o_j$. We call $\ldg$ a \emph{labeled dependency graph}. 
The labeled dependency graph for Fig.~\ref{fig:ex-prob}(a) is given in 
Fig.~\ref{fig:ex-prob}(b).
We can immediately identify multiple circular dependencies in the graph, 
e.g., between objects $3$ and $5$, or among objects $7, 2, 6$ and $5$. 
It is not difficult to see that the dependency graph abstraction fully 
captures the information needed to solve a tabletop rearrangement problem 
with external buffers. 

\vspace{-1mm}
\subsection{Unlabeled Tabletop Rearrangement Problem}
\vspace{-1mm}
In an unlabeled setting, objects are interchangeable. That is, it does not 
matter which object goes to which goal. For example, in Fig.~\ref{fig:ex-prob}, 
object $5$ can move to the goal for object $6$. We call this version the 
Unlabeled Running Buffer Minimization (\urbm) problem, which is intuitively 
easier. The plan for the unlabeled problem can be represented similarly as the 
labeled setting; we continue to use labels but do not require matching labels 
for start and goal poses.

For the unlabeled setting, there clearly remains dependency between start 
and goal arrangements, but in a different form. We update the \emph{unlabeled} dependency graph 
for \urbm as an undirected \emph{bipartite} graph between the start 
arrangement and the goal arrangement. That is, $\udg = (V_1\cup V_2, E)$ 
where each $v \in V_1$ (resp., $v \in V_2$) corresponds to a start (resp., goal) 
pose $p \in \mathcal A_1$ (resp., $p \in \mathcal A_2$). We denote the vertices representing the start and goal poses as \emph{start vertices} and \emph{goal 
vertices}, respectively. There is an edge between 
$v_1 \in V_1$ and $v_2 \in V_2$ if the objects at the corresponding poses 
overlap. The unlabeled dependency graph for Fig.~\ref{fig:ex-prob}(a) is given 
in Fig.~\ref{fig:ex-prob}(c). 

We make a straightforward but important observation of the unlabeled dependency 
graph when objects are uniform cylinders, which is a key sub-class of 
\toro problems, e.g., many products can be approximated as uniform cylinders. 
\vspace{-1mm}
\begin{proposition}\label{p:udg}
For unlabeled tabletop object rearrangement problems where all objects are identical 
cylinders, the unlabeled dependency graph is a planar bipartite graph with maximum 
degree $5$. 
\end{proposition}
\begin{proof}
The bipartite and planar part come directly from the problem setup. Since we 
work with uniform cylinders which have uniform disc base, one disc may only 
touch six non-overlapping discs and non-trivially intersect at most five 
non-overlapping discs. 
\end{proof}
%\vspace{-2mm}
%$\begin{proof}See Sec.~\ref{sec:proofs}.\end{proof}
%\vspace{-2mm}

\section{Structural Analysis and NP-Hardness}\label{sec:structure}
In this section, we highlight some important structural properties of \lrbm, including 
(1) the comparison to minimizing the total number of buffers \cite{han2018complexity}, 
(2) the solutions of \lrbm and the \emph{linear arrangement}\cite{shiloach1979minimum} or 
\emph{linear ordering} \cite{adolphson1973optimal} of its dependency graph, and (3) the 
hardness of computing \mrb for labeled dependency graphs. 

\vspace{-2mm}
\subsection{Running Buffer versus Total Buffer}
As mentioned in the introduction, running buffers are related to but 
different from the total number of buffers required, as studied in 
\cite{han2018complexity}, to solve a rearrangement problem using 
external buffers. It was shown that the minimum number of total 
buffers for solving an \lrbm is the same as the size of the minimum 
\emph{feedback vertex set} (FVS) of the underlying dependency graph. 
An FVS is a set of vertices the removal of which leaves a graph 
acyclic. An \lrbm with an acyclic dependency graph can be solved 
without using any buffer. We denote the size of the minimum FVS as 
\fvs. 

As an example, for \lrbm, consider a labeled dependency graph 
that is formed by $n$ copies of $2$-cycles. The \fvs is $n$. On the 
other hand, the \mrb is just $1$ for the problem. That is, only a 
single external buffer is needed to solve the problem. Therefore, 
whereas the total number of buffers used has more bearing on global 
solution optimality, \mrb sheds more light on  \emph{feasibility}. 
Knowing the \mrb tells us whether a certain number of external 
buffers will be sufficient for solving a class of rearrangement problems. 
This is critical for practical applications where the number of 
external buffers is generally limited to be a small constant. 

We give an example where the \mrb and \fvs cannot always be optimized 
simultaneously. 
%\kg{The reviewers show confusion here. We'd better clearer demonstrate the purpose of this example. The sentence in the brackets is the original sentence.}
%[We give an example where if the \mrb is 
%optimized first, the total number of buffers used is larger than the \fvs.]
For the setup (Fig.~\ref{fig:ex-mrb-fvs}) where objects have convex
footprints, the \fvs, $\{7, 9, 10\}$, has size $3$. Using our algorithms, to be 
detailed later, the \mrb is $2$ (e.g., with the sequence $10, 8, 4, 5, 3, 6, 7, 1,
 2, 9$, the interpretation of which is given in Sec.~\ref{subsec:logvrb}).
However, constrained on $\mrb=2$, the total number of buffers that 
must be used is at least $4 > 3 $. We note that, this is rarely the case; for 
uniform cylinders, the total number of buffers needed after first minimizing 
the running buffer is almost always the same as the \fvs size. 
\begin{figure}[h!]
    \centering
    \vspace{2mm}
    \begin{overpic}[height=1.2in]{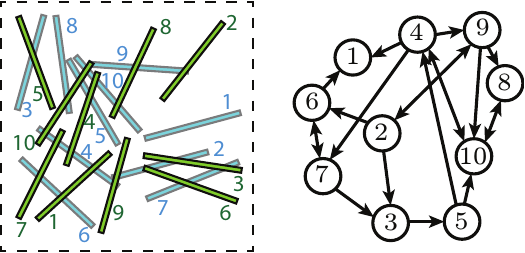}
    \end{overpic}
    \vspace{1mm}
    \caption{An \lrbm instance with uniform thin cuboids (left) and its 
    labeled dependency graph, where the total number of buffers needed is 
    more than the size of the \fvs when the number of running buffers is 
    minimized.}
    \label{fig:ex-mrb-fvs}
\end{figure}
%\jy{ toto: higher quality figure.}

\vspace{-2mm}
\subsection{Linear Ordering of Graph Vertices and Running Buffer}\label{subsec:logvrb}
Given a graph with vertex set $V$, a \emph{linear ordering} of $V$ is a bijective 
function $\varphi: \{1, \ldots, |V|\} \to V$. Given a dependency graph $G = (V, A)$ 
for an \lrbm and a linear ordering $\varphi$, we may turn it into a plan 
$P$ by sequentially picking up objects corresponding to vertices $\varphi(1), \varphi(2), 
\ldots$ For each object that is picked up, it is moved to its goal pose if it has no 
further dependencies; otherwise, it is stored in the buffer. Objects already in the 
buffer will be moved to their goal pose at the earliest possible opportunity. 

For example, given the linear ordering $1, 5, 6, 3, 4, 2, 7$ for the dependency 
graph from Fig.~\ref{fig:ex-prob}(b), first, $o_1$ can be directly moved to its 
goal. Then, $o_5$ is moved to the buffer because it has dependency on $o_3$ and $o_7$ 
(but no longer on $o_1$). Then, $o_6$ can be directly moved to the buffer because 
$o_5$ is now at a buffer location. Similarly, $o_3$ can be moved to its goal next. 
Then, $o_4$ and $o_2$ must be moved to buffer, after which $o_7$ can be moved to its
goal directly. Finally, $o_2, o_4$, and $o_5$ can be moved to their respective goals 
from the buffer. This leads to a maximum running buffer size of $3$. This 
is not optimal; an optimal sequence is $5, 6, 2, 7, 4, 3, 1$, with $\mrb = 2$. Both 
sequences are illustrated in Fig.~\ref{fig:lr}.
\begin{figure}[h!]
\vspace{1mm}
    \centering
    \begin{overpic}
    [width=0.45\textwidth]{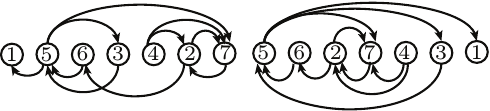}
    \put(22, -3){{\footnotesize (a)}}
    \put(74, -3){{\footnotesize (b)}}
    \end{overpic}
    \vspace{3mm}
    \caption{Two linear orderings of vertices of the labeled 
    dependency graph from Fig.~\ref{fig:ex-prob}(b). The right one minimizes \mrb.}
    \label{fig:lr}
\end{figure}

From the discussion, we may view the number of running buffers as a function of 
a dependency graph $G$ and a linear ordering $\varphi$, i.e., $\rb(G, \varphi)$ is 
the number of running buffers needed for rearranging $G$ following the order 
given by $\varphi$. We then have $\mrb(G) = \min_{\varphi}\rb(G, \varphi)$.

\subsection{Intractability of Computing $\mrb(G)$}
Since computing \fvs is NP-hard \cite{han2018complexity}, one would expect that computing \mrb 
for a labeled dependency graph, which can be any directed graph, is also hard. We show 
that this is indeed the case, through examining the interesting relationship between 
\mrb and the \emph{vertex separation problem} (\vsp), which is equivalent to path 
width, gate matrix layout and search number problems as described in Theorem 3.1 in 
\cite{diaz2002survey}, resulting from a series of studies 
\cite{kirousis1986searching, kinnersley1992vertex, fellows1989search}.
 Unless $P=NP$, there cannot be an absolute approximation algorithm
for any of these problems \cite{bodlaender1995approximating}. First, we describe the 
vertex separation problem. 
Intuitively, given an undirected graph $G = (V, E)$, \vsp seeks a linear ordering $\varphi$ 
of $V$ such that, for a vertex with order $i$, the number of vertices come no later than 
$i$ in the ordering, with edges to vertices that come after $i$, is minimized. 
%

% \begin{problem}[VertSep]
% \end{problem}

\vspace{1mm}
\noindent\fbox{\begin{minipage}{3.3in}
\vspace{1mm}
\noindent
\textbf{Vertex Separation (\vsp)}\\
\noindent
\textbf{Instance}: Graph $G(V,E)$ and an integer $K$.\\
\noindent
\textbf{Question}: Is there a bijective function $\varphi: \{1, \dots, n\} \to 
V$, such that for any integer $1 \le i \le n$, 
$|\{u\in V \mid \exists (u, v) \in E \ and \ \varphi(u)\leq i < \varphi(v) \}| \leq K$?
\vspace{1mm}
\end{minipage}}
\vspace{1mm}

As an example, in Fig.~\ref{fig:vsp}(a), with the given linear ordering, at the second 
vertex, both the first and the second vertices have edges crossing the vertical separator, 
yielding a crossing number of $2$. Given a graph $G$ and a linear ordering $\varphi$, we 
define $\vertsep(G, \varphi) :=
\max_i | \{u\in V \mid \exists (u, v) \in E  \ and \  \varphi(u)\leq i < \varphi(v) 
\} |$, \vsp seeks $\varphi$ that minimizes $\vertsep(G, \varphi)$. 
Let $\minvs(G)$, the vertex separation number of graph $G$, be the minimum $K$ for
which a \vsp instance has a yes answer, then $\minvs(G) = \min_\varphi \vertsep(G, \varphi)$.
Now, given an undirected
graph $G$ and a labeled dependency graph $\ldgg$ obtained from $G$ by replacing each 
edge of $G$ with two directed edges in opposite directions, we observe that there 
are clear similarities between $\vertsep(G, \varphi)$ and $\rb(\ldgg, \varphi)$, which 
is characterized in the following lemma. 

\begin{lemma}
$\vertsep(G, \varphi) \leq \rb(\ldgg, \varphi) \leq \vertsep(G, \varphi) + 1$.
\end{lemma}
\begin{proof}[Proof sketch]
Fixing a linear ordering  $\varphi$, it is clear that $\vertsep(G, \varphi) \le 
\rb(\ldgg, \varphi)$, since the vertices on the left side of a separator with edges 
crossing the separator for $G$ corresponds to the objects that must be stored 
at buffer locations. For example, in Fig.~\ref{fig:vsp}(a), past the second vertex
from the left, both the first and the second vertices have edges crossing the 
vertical ``separator''. In the corresponding dependency graph shown in 
Fig.~\ref{fig:vsp}(b), objects corresponding to both vertices must be moved 
to the external buffer. 
%
%Intuitively, this is the case because, for any vertex separator
%
On the other hand, we have $\rb(\ldgg, \varphi) \leq \vertsep(G, \varphi) + 1$ 
because as we move across a vertex in the linear ordering, the corresponding 
object may need to be moved to a buffer location temporarily. 
For example, as the third vertex from the left in Fig.~\ref{fig:vsp}(a) is 
passed, the vertex separator drops from $2$ to $1$, but for dealing with 
the corresponding dependency graph in Fig.~\ref{fig:vsp}(b), the object 
corresponding to the third vertex from the left must be moved to the 
buffer before the first and the second objects stored in buffer can be 
placed at their goals. 
\end{proof}

\begin{figure}[h!]
\vspace{-1mm}
    \centering
\begin{overpic}
[width=0.90\columnwidth]{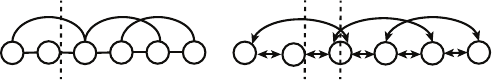}
\put(19, -3){{\footnotesize (a)}}
\put(72, -3){{\footnotesize (b)}}
\end{overpic}
\vspace{3mm}
\caption{(a) An undirected graph and a linear ordering of its vertices. (b)
A corresponding labeled dependency graph with the same vertex ordering.}
\label{fig:vsp}
\vspace{-2mm}
\end{figure}

\begin{theorem}
Computing \mrb, even with an absolute approximation, for a labeled dependency graph is NP-hard. 
\end{theorem}
\vspace{-1mm}
\begin{proof}
Given an undirected graph $G$, we reduce from approximating \vsp within a 
constant to approximating \mrb within a constant for a dependency 
graph  $\ldgg$ from $G$ constructed as stated before, replacing each edge in $G$ as a 
bidirectional dependency. 

Unless $P=NP$, \vsp does not have absolute approximation in polynomial time. 
Henceforth, if $\mrb(\ldgg, \varphi)$  can be approximated within $\alpha$ in polynomial time, 
which means 
for graph $G$, we can find a $\varphi^*$ in polynomial time such that $\rb(\ldgg, \varphi^*) 
\leq  \mrb(\ldgg) + \alpha$, we then have $\vertsep(G, \varphi^*)\leq \rb(\ldgg, \varphi^*)  
\leq \alpha + \mrb(\ldgg) \leq \minvs(G)+\alpha+1$, which shows vertex separation can have 
an absolute approximation, implying $P=NP$.
\end{proof}

\begin{comment}
\begin{itemize}
    \item Hardness result for dependency graphs - relation to vertex separation. 
\end{itemize}
\end{comment}

%\begin{figure}[ht]
%\begin{center}
%\begin{overpic}[width={\iftwocolumn 3.5in \else 5in \fi},tics=5]
%{../figures/problem.eps}
%\put(7, 3){{\small $(x_L, 0)$}}
%\put(83, 33){{\small $(x_R, y_T)$}}
%\put(29, -2){{\small $(x_A, 0)$}}
%\put(45.5,-2){{\small $(0, 0)$}}
%\put(15,28){{\small moving direction}}
%\end{overpic}
%\end{center}
%\caption{\label{fig:conveyor}Illustration of a conveyor workspace where 
%the base of the robot arm is located at $(X_A,0)$. The end-effector %picks 
%up objects within a region $\W$ with a lower left corner of $(x_L, 0)$ 
%and an upper right corner of $(x_R, y_T)$, and drops off objects at 
%the drop-off location $(0, 0)$.}
%\end{figure}

\section{Lower and Upper Bounds on \mrb}\label{sec:bounds}
We proceed to establish bounds on \mrb, i.e., what is the lowest 
possible \mrb for \lrbm and \urbm, and what is the best that we can do to lower 
\mrb? 
An important outcome is that \mrb can grow unbounded with the number of objects, 
even for \urbm when objects are all uniform cylinders. Another very interesting 
result is that we are able to close the gap between lower and upper bound for 
\urbm for uniform cylinders. 

\vspace{-1mm}
\subsection{Intrinsic \mrb Lower Bounds} 
\vspace{-1mm}
When there is no restrictions on object footprint, \mrb can easily reach the maximum 
possible $n -1$ for an $n$ object instance, even in the \urbm case. Such an example is 
given in Fig.~\ref{fig:stick}, where $n = 6$ thin cuboids are aligned horizontally 
in $\mathcal A_1$, one above the other. The cuboids are vertically aligned in 
$\mathcal A_2$, and every pair of start pose and goal pose induces a collision.
Clearly, this yields a bidirectional $K_6$ labeled dependency graph in the \lrbm 
case and a $K_{6, 6}$ unlabeled dependency graph in the \urbm case. For both, 
$n - 1 = 5$ objects must be moved to buffer before the problem can be resolved. 
The example clearly can be generalized to an arbitrary number of objects. 

\vspace{-1mm}
\begin{proposition}
\mrb lower bound is $n - 1$ for $n$ objects for both \lrbm and \urbm, which is 
the maximum possible, even for uniform convex shaped objects. 
\end{proposition}
\vspace{-1mm}

\begin{figure}[h!]
    \centering
    \includegraphics[width=0.15\textwidth]{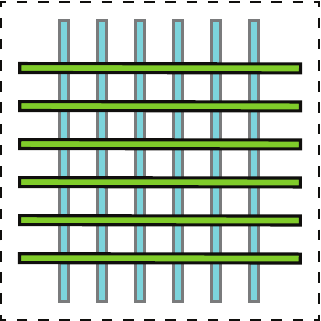}
    \vspace{1mm}
    \caption{An instance with 6 cuboids where horizontal and vertical sets represent 
    start and goal poses, respectively.}
    \label{fig:stick}
\end{figure}
%\jy{The figure can be updated to be thin rectangles/cuboids.}

The lower bound on \mrb being $\Omega(n)$ is undesirable, but it is established using 
objects that are ``thin''. Everyday objects are not often like that. An ensuing question of 
high practical value is then: what happens when the footprint of the objects are ``nicer''? 
Next, we show that, the lower bound drops to $\Omega(\sqrt{n})$ for uniform cylinders, 
which approximate many real-world objects/products. Further more, we show that this lower 
bound is tight for \urbm (in Section~\ref{subsec:upper}).

We first establish the $\Omega(\sqrt{n})$ lower bound for \urbm. 
For convenience, assume $n$ is a perfect square, i.e., $n = m^2$ for some integer 
$m$. 
To get to the proof, a grid-like unlabeled dependency graph is used, which we call 
a \emph{dependency grid}, where $\mathcal A_1$ and $\mathcal A_2$ have fixed grid
(rotated by $\pi/4$) patterns, an example of which is given in Fig.~\ref{fig:DependencyGrid}. We use 
$\D(w, h)$ to denote a dependency grid with $w$ columns and $h$ rows. Let $(x, y)$ 
be the coordinate of a vertex $v_{x,y}$ on $\D(w, h)$ with the top left being $(1,
1)$. The parity of $x + y$ determines the partite set of the vertex (recall that 
unlabeled dependency graph for uniform cylinders is always a planar bipartite 
graph, by Proposition~\ref{p:udg}), which may correspond to a start pose or 
a goal pose. With this in mind, we simply call vertices of $\D(w, h)$ start and 
goal vertices; let $v_{1,1}$ be a start vertex. 
%
%Assuming the cylinders' footprints have unit radius, the distance between adjacent 
%vertices on $\D(w, h)$ is $\sqrt{2}$.
%

%For the unlabeled setting, we show that the \mrb of the worst cases is $\Omega(\sqrt{n})$ where $n$ is the number of objects in the instance. This property is proved with a instance whose dependency graph is a \emph{dependency grid} whose definition is as follows.

%\begin{definition}[Dependency Grid]
%A $W\times H$ dependency grid $\mathcal{D}$ is a square grid graph whose vertices correspond to points in the plane with integer coordinates, $x$-coordinates being in the range 1, ..., $W$, $y$-coordinates being in the range 1, ..., $H$ (Fig. \ref{fig:DependencyGrid}). Each vertex of $\mathcal{D}$ is colored blue or black but the neighboring vertices cannot be with the same color. Denote the vertex at the $i$th column and $j$th row in $\mathcal{D}$ as $v_{i,j}$. Without loss of generality, we can assume that $v_{1,1}$ is black.
%\end{definition}

\begin{figure}[h!]
    \centering
    \includegraphics[width=0.35\textwidth]{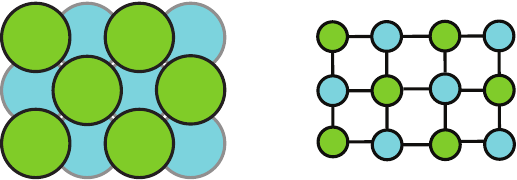}
    \vspace{1mm}
    \caption{A \urbm instance (left) and its unlabeled dependency graph (right), a 
    $4\times 3$ dependency grid. Green and cyan indicate start and goal 
    arrangements, respectively.}
    \label{fig:DependencyGrid}
\end{figure}

%When the poses in $\mathcal{A}_1$ and $\mathcal{A}_2$ are placed at the coordinates of black and blue vertices in $\mathcal{D}$ respectively and the unit length of the grid is $\sqrt{2}$ units, $\mathcal{D}$ is the dependency graph of the rearrangement instance.

We use $\D(m, 2m)$ for establishing the lower bound on \mrb. We use a \emph{vertex 
pair} $p_{i, j}$ to refer to two adjacent vertices $v_{i, 2j-1}$ and $v_{i, 2j}$ in 
$\D(m, 2m)$. It is clear that a vertex pair contains a start and a goal vertex. 
We say that a goal vertex is \emph{filled} if an object is placed at the corresponding goal pose. 
We say that a start vertex (which belongs to a vertex pair) is \emph{cleared} if the 
corresponding object at the vertex is picked (either put at a goal or at a buffer)
but the corresponding goal in the vertex pair is not filled. 
At any moment when the robot is not holding an object, the number of objects in the 
buffer is the same as the number of cleared vertices. For each column $i, 1\leq i 
\leq m$, let $f_i$ (resp., $c_i$) be the number of goal (resp., start) vertices in 
the column that are filled (resp., cleared). Notice that a goal cannot be filled 
until the object at the corresponding start vertex is removed.

%In the rest of the subsection, we prove that for a rearrangement instance whose dependency graph is a $\sqrt{n}\times (2\sqrt{n})$ dependency grid $\mathcal{D}$, the minimum \mrb is $\Omega (\sqrt{n})$. The general idea is to evaluate the minimum needed buffer space when $\lfloor n/3\rfloor$ objects are at the goal poses.
%
%Without loss of generality, we can assume that $n$ is a square number. Note that in each column, there are $m$ start vertices and goal vertices respectively. A \emph{vertex pair} $P(i,j)$ is defined to be a pair of adjacent vertices in the same column ($v_{i,2j-1}$, $v_{i,2j}) (1\leq i,j \leq m)$. One of them is a start vertex and the other one is a goal vertex. A goal vertex is \emph{filled\ up} if there is an object at the corresponding pose. A \emph{cleared vertex} is a start vertex whose corresponding object has been moved away but the goal vertex in the same vertex pair is not filled up. Therefore, when $\lfloor n/3\rfloor$ goal vertices are filled up, the number of objects in the buffer is the same as the number of cleared vertices at the moment. For each column $i(1\leq i \leq m)$, denote the number of filled-up goal vertices and the number of cleared vertices as $g_i$ and $p_i$ respectively.

\begin{lemma}\label{lemma:adjacent}
On a dependency grid $\D(m, 2m)$, for two adjacent columns $i$ and $i+1$, $1 \le i < m$, if
$f_i+f_{i+1}\neq 0$ or $2m$, then $c_i+c_{i+1}\geq 1$. In other words, there is at least one 
cleared vertex in the two adjacent columns unless $f_i=f_{i+1}=0$ or $f_i=f_{i+1}=m$.
\end{lemma}
%\begin{proof}See Sec.~\ref{sec:proofs}.\end{proof}
\begin{proof}
If there is a $j, 1\leq j \leq m$, such that only one of the goal vertices in vertex pairs 
$p_{i,j}$ and $p_{i+1, j}$ is filled (Fig. \ref{fig:LemmaProof}(a)), then the start 
vertex in the other vertex pair must be cleared. Therefore, $c_i+c_{i+1}\geq 1$.

On the other hand, if, for each $j, 1\leq j \leq m$, both or neither of the goal vertices 
in $p_{i,j}$ and $p_{i+1,j}$ is filled, then there is a $j, 1 \leq j \leq m-1$, such that both 
goal vertices in $p_{i,j}$ and $p_{i+1,j}$ are filled but neither of those in $p_{i,j+1}$ and $p_{i+1,j+1}$ is filled (Fig.~\ref{fig:LemmaProof}(b)) or the opposite 
(Fig.~\ref{fig:LemmaProof}(c)). Then, for the vertex pairs whose goal vertices are not filled, 
say $p_{i,j+1}$ and $p_{i+1, j+1}$, one of their start vertices is a neighbor of the filled 
goal in $p_{i,j}$ and $p_{i,j+1}$. Therefore, at least one of the start vertices in $p_{i, j+1}$ 
and $p_{i+1,j+1}$ is a cleared vertex. And thus, $c_i+c_{i+1}\geq 1$.

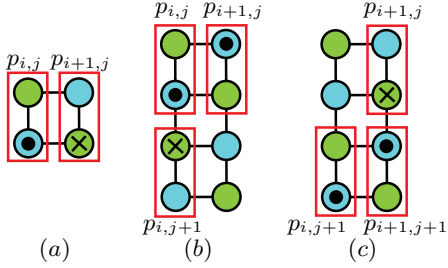
\begin{figure}[h!]
    \centering
    \vspace{5mm}
    \begin{overpic}
    [width = 0.30\textwidth]{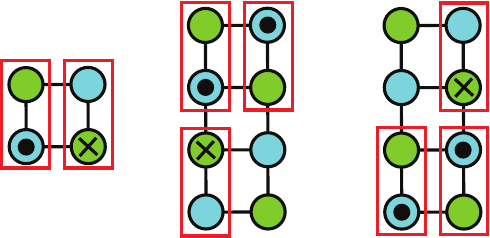}
    \put(0, 38){{\small $p_{i,j}$}}
    \put(11, 38){{\small $p_{i+1,j}$}}
    
    \put(37, 51){{\small $p_{i,j}$}}
    \put(49, 51){{\small $p_{i+1,j}$}}
    \put(34, -3){{\small $p_{i,j+1}$}}
    
    \put(89, 51){{\small $p_{i+1,j}$}}
    \put(70, -3){{\small $p_{i,j+1}$}}
    \put(89, -3){{\small $p_{i+1,j+1}$}}
    
    \put(8, -9.5){{\small $(a)$}}
    \put(44, -9.5){{\small $(b)$}}
    \put(85, -9.5){{\small $(c)$}}
    \end{overpic}
    \vspace{6mm}
    \caption{Some cases discussed in the lemma~\ref{lemma:adjacent}. The green and cyan nodes represent the start and goal vertices in the dependency graph. Specifically, the cyan nodes with a dot inside represent the filled vertices and the green nodes with a cross inside represent the cleared vertices. (a) When only one goal vertex in $p_{i,j}$ and $p_{i+1, j}$ is filled up, the start vertex in the other vertex pair is a cleared vertex. (b) When both goal vertices in $p_{i,j}$ and $p_{i+1,j}$ are filled but neither of those in $p_{i,j+1}$ and $p_{i+1,j+1}$ is filled, one of the start vertices $p_{i,j+1}$ and $p_{i+1,j+1}$ is a cleared vertex. (c) The opposite case of (b).}
    \label{fig:LemmaProof}
\end{figure}
\end{proof}

\begin{lemma}\label{l:urbm-lower}
Given a \urbm instance with $n = m^2$ objects and whose dependency graph is $\D(m, 2m)$, 
its \mrb is lower bounded by $\Omega(m) = \Omega(\sqrt{n})$
\end{lemma}
%\begin{proof}See Sec.~\ref{sec:proofs}.\end{proof}
\begin{proof}%[Proof of Lemma~\ref{l:urbm-lower}]
We show that there are $\Omega(m)$ cleared vertices when $\lfloor n/3 \rfloor$ goal vertices 
are filled. Suppose there are $q$ columns in $\mathcal{D}$ with $1\leq f_i\leq m-1$. According 
to the definition of $f_i$, for each of these $q$ columns, there is at least one goal vertex 
that is filled and at least one goal vertex that is not.

If $q< \dfrac{\lfloor n/3 \rfloor}{3(m-1)}$, then there are two columns $i$ and $j$, such that $f_i=m$ and $f_j=0$. That is 
because $\sum_{1\leq i \leq m} f_i= \lfloor n/3 \rfloor$ and $0\leq f_i \leq m$ for all $1\leq i \leq m$. 
Therefore, for the vertex pairs in each row $j$, at least one goal vertex is filled 
but at least one is not. And thus, for each $j$, there are two adjacent columns $i, i+1, 1\leq 
i < m$, such that there is only one goal vertex in $p_{i,j}$ and $p_{i+1,j}$ is filled and the
start vertex in the other vertex pair is cleared (Fig. \ref{fig:LemmaProof}(a)). 
Therefore, there are at least $m$ cleared vertices in this case.

If $q\geq \dfrac{\lfloor n/3 \rfloor}{3(m-1)}$, then we partition all the columns in $\mathcal{D}$ into $\lfloor m/2\rfloor$ disjoint pairs: 
(1,2), (3,4), ... The $q$ columns belong to at least $\lfloor q/2\rfloor$ pairs of 
adjacent columns. Therefore, according to Lemma \ref{lemma:adjacent}, we have $\Theta(m)$ 
cleared vertices.

In conclusion, there are $\Omega(m)$ cleared vertices when there are $\lfloor n/3 \rfloor$ 
filled goal. Therefore, the minimum \mrb of this instance is $\Omega(m)$.
\end{proof}

Because a \urbm always have lower \mrb than an \lrbm with the same objects and 
goal placements, the conclusion of Lemma~\ref{l:urbm-lower} directly applies to 
\lrbm. Therefore, we have 

\begin{theorem}
For both \urbm and \lrbm with $n$ uniform cylinders, \mrb is lower bounded by 
$\Omega(\sqrt{n})$.
\end{theorem}

For uniform cylinders, while the lower bound on \urbm is tight (as shown in
Section~\ref{subsec:upper}), we do not know whether the lower bound on \lrbm is tight;
our conjecture is that $\Omega(\sqrt{n})$ is not a tight lower bound for \lrbm. 
Indeed, the $\Omega(\sqrt{n})$ lower bound can be realized when uniform cylinders are 
simply arranged on a cycle, an illustration of which is given in Fig.~\ref{fig:lrbm-cycle}. 
For a general construction, for each object $o_i$, let $o_i$ depend on $o_{(i-1\mod n)}$ 
and $o_{(i+\sqrt{n}\mod n)}$, where $n$ is the number of objects in the instance. From the 
labeled dependency graph, we can construct the actual \lrbm instance where start and 
goal arrangements both form a cycle. 
We can show that when $n/2$ objects are at the goal poses, $\Omega(\sqrt{n})$ objects are 
at the buffer. We omit the proof, which is similar in spirit to that for Lemma~\ref{l:urbm-lower}. %We consider each $mod \sqrt{n}$ class of objects $M_i=\{o_j:j\equiv i mod n\}$. Denote the number of objects in $M_i$ at the goal poses by $g_i$. Suppose that there are $m$ $mod \sqrt{n}$ classes with $1\leq g_i \leq \sqrt{n}-1$. Similar to the proof of Thm. \ref{thm:unlabeled}, no matter $m=o(\sqrt{n})$ or $\Theta(\sqrt{n})$, the number of objects at the buffer is $\Omega(\sqrt{n})$.
\begin{figure}[h!]
    \centering
    \includegraphics[height=1.1in]{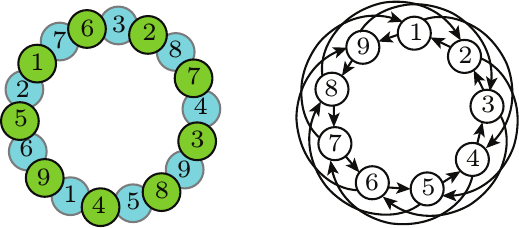}
    \caption{An example of a 9-object \lrbm yielding $\Omega(\sqrt{n})$ \mrb (left)
    and the corresponding dependency graph (right).}
    \label{fig:lrbm-cycle}
\end{figure}

\vspace{-2mm}
\subsection{\mrb Algorithmic Upper Bounds}\label{subsec:upper}
We now establish that regardless of how $n$ uniform cylinders are to be 
rearranged, the corresponding \urbm instance admits a solution using 
$O(\sqrt{n})$ running buffers. The lower and upper bounds on \urbm agree;
therefore, the $O(\sqrt{n})$ bound is tight. 

To prove the upper bound, We propose an $O(n\log(n))$-time algorithm \spp for the 
setting based on a vertex separator of $\udg$.
\spp computes a sequence of goal vertices to be removed from the dependency graph.
Given a sequence of goal vertices to be removed, 
the running buffer size at each moment equals
$$
\max(0,\|N(g,\udg)\|-\|g\|)
$$

where $g$ is the set of removed goal vertices at this moment,
and $N(g,\udg)$ is the set of neighbors of $g$ in $\udg$.
We prove that \spp can find a rearrangement plan 
with $ O(\sqrt{n})$ running buffers.

\begin{algorithm}
\begin{small}
    \SetKwInOut{Input}{Input}
    \SetKwInOut{Output}{Output}
    \SetKwComment{Comment}{\% }{}
    \caption{ \spp}
		\label{alg:spp}
    \SetAlgoLined
		\vspace{0.5mm}
    \Input{$\udg(V,E)$: unlabeled dependency graph}
    \Output{$\pi$: goal sequence}
		\vspace{0.5mm}
		$\pi,V,E\leftarrow$ RemovalTrivialGoals($\udg(V,E)$)\\
        \lIf{$V$ is $\emptyset$}{
        \Return $\pi$
        }
        $A, B, C \leftarrow$ Separator($\udg(V,E)$)\\
        $\pi \leftarrow \pi + g(C)$\\
        $A' \leftarrow A-N(g(C),\udg(V,E))$\\
        $B' \leftarrow A-N(g(C),\udg(V,E))$\\
		$\pi_{A'}\leftarrow $ \spp($\udg(A',E(A'))$)\\
		$\pi_{B'}\leftarrow $ \spp($\udg(B',E(B'))$)\\
		\lIf{$|g(A')|-|s(A')|\geq |g(B')|-|s(B')|$}{
		    $\pi\leftarrow \pi + \pi_{A'} + \pi_{B'}$
		}
		\lElse{
		$\pi\leftarrow \pi + \pi_{B'} + \pi_{A'}$
		}
		\Return $\pi$\\
\end{small}
\end{algorithm}

The algorithm is presented in Algo.~\ref{alg:spp}. 
\spp consumes a graph $\udg(V,E)$, which is a subgraph of $\udg$ induced by vertex set $V$.
To start with, the isolated goal vertices or those with only one dependency in $\udg(V,E)$ can be removed without using buffers (Line 1).
After that, $V$ can be partitioned into three disjointed subsets $A$, $B$ and $C$ \cite{lipton1979separator} (Line 3), 
such that there is no edge connecting vertices in $A$ and $B$, $|A|,|B|\leq 2|V|/3$, and $|C|\leq 2\sqrt{2|V|}$ (Fig.~\ref{fig:separator}(a)). 
For the start vertices in $C$ and the neighbors of the 
goal vertices in $C$, we remove them from $\udg$. 
Since there are at most $5$ neighbors for each goal vertex, there are at most $10\sqrt{2|V|}$ objects moved to the buffer in this operation. 
After that, we remove the goal vertices 
in $C$,
which should be isolated now (Line 4). 
Function $g(\cdot)$ obtains the goal vertices in a given vertex set.
Let $A'$, $B'$ be the remaining vertices in $A$ and $B$ (Line 5-6). 
Function $N(\cdot,\cdot)$ obtains the neighbors of a vertex set in a given dependency subgraph.
With the removal of $C$ from $\udg$, $A'$ and $B'$ form two independent subgraphs (Fig.~\ref{fig:separator}(b)). 
We can deal with the subgraphs one after the other by 
recursively calling \spp(Fig.~\ref{fig:separator}(c)) (Line 7-8). 
Let $\delta(V'):= |g(V')|-|s(V')|$ where $g(V')$ and $s(V')$ are the goal and start vertices in a vertex set $V'$ respectively. 
Between vertex subsets $A'$ and $B'$, we prioritize the one with larger $\delta(\cdot)$ value(Line 9-10).

\begin{figure}[h!]
    \centering
\begin{overpic}
[width=0.45\textwidth]{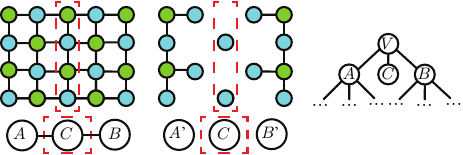}
\put(12,-5){(a)}
\put(46,-5){(b)}
\put(82,-5){(c)}
\end{overpic}
\vspace{4mm}
    \caption{The recursive solver \spp for \urbm. (a) A $O(\sqrt{|V|})$ vertex separator for the planar dependency graph. (b) By removing the start vertices in $C$ and the neighbors of the goal vertices in $C$, the remaining graph consists of two independent subgraphs and isolated goal poses in $C$. (c) The problem can be solved by recursively calling \spp.}
    \label{fig:separator}
\end{figure}

To investigate the running buffer size of the plan computed by \spp,
we construct a binary tree $T$ based on \spp(Fig.~\ref{fig:separator}(c)).
Each node represents a recursive call consuming a subgraph $\udg(V,E)$ and the left and right children of the node are induced from subgraphs $\udg(A',E(A'))$ and $\udg(B',E(B'))$ of $\udg(V,E)$.
\spp computes the plan by visiting the binary tree in a depth-first manner.
For each node on $T$, given the input dependency graph $\udg(V,E)$, 
denote the sequence before we deal with $\udg(V,E)$ as $\pi_0$. 
Denote the vertices pruned in RemovalTrivialGoals( $\udg(V,E)$) as $P$.
Without loss of generality,
assume that \spp recurses into $A'$ before $B'$.
Let $\pi_{P}$, $\pi_{C'}$, $\pi_{A'}$ and $\pi_{B'}$ be the goal removal sequence after we remove vertices in $P$, $C'$, $A'$ and $B'$ from $\udg(V,E)$ respectively.
We define function $RB(\pi)$ to represent ``generalized'' current running buffer size of a removal sequence $\pi$:
When vertices in $\pi$ are removed from the dependency graph, 
if the current running buffer size is positive, 
then $RB(\pi)$ equals the number of running buffers; 
otherwise, $RB(\pi)$ equals the negation of the number of empty goal poses.
In the depth-first recursion, each node on the binary tree $T$ is visited at most three times:

\begin{enumerate}
    \item Before exploring child nodes. The peak may be reached during the removal of $C'$. Let $\pi_{C^*}$ be the sequence when the running buffer size reaches the peak. 
    $RB(\pi_{C^*}) \leq RB(\pi_0) + 10\sqrt{2} \sqrt{|V|}$.
    \item After we deal with $A'$ and before we deal with $B'$.
    $RB(\pi_{A'}) = RB(\pi_{0}) + (s(C')-g(C')) + (s(A')-g(A'))$.
    \item After we deal with $B'$.
    $RB(\pi_{B'}) = RB(\pi_0) + (s(V)-g(V))$.
\end{enumerate}

\begin{lemma}\label{lemma:average}
$RB(\pi_{A'})\leq (RB(\pi_{C^*})+RB(\pi_{B'}))/2$.
\end{lemma}

\begin{proof}
\begin{equation*}
    \begin{split}
        & RB(\pi_{A'}) \\
        & = RB(\pi_{0}) + (s(P)-g(P)) + (s(C')-g(C'))\\
        &\ \ \ \ + (s(A')-g(A'))\\ 
%     \end{split}
% \end{equation*}
% \begin{equation*}
%     \begin{split}
        &= RB(\pi_0) + (s(P)-g(P)) + (s(C')-g(C')) \\
        &\ \ \ \ + \min[s(A')-g(A'), s(B')-g(B')]\\
%     \end{split}
% \end{equation*}
% \begin{equation*}
%     \begin{split}
        &\leq RB(\pi_0) + (s(P)-g(P)) + (s(C')-g(C'))\\ 
        &\ \ \ \  +\dfrac{1}{2}[(s(V)-g(V)) - (s(P)-g(P))\\
        &\ \ \ \ - (s(C')-g(C'))]\\
%     \end{split}
% \end{equation*}
% \begin{equation*}
%     \begin{split}
        &= \dfrac{1}{2} \{[RB(\pi_0) + s(V)-g(V)]\\
        &\ \ \ \ +[RB(\pi_0) + (s(P)-g(P)) + (s(C')-g(C'))]\}\\
%     \end{split}
% \end{equation*}
% \begin{equation*}
%     \begin{split}
        &= \dfrac{1}{2} [RB(\pi_{B'})+RB(\pi_{C'})]\\
        &\leq \dfrac{1}{2} [RB(\pi_{B'})+RB(\pi_{C^*})]
    \end{split}
\end{equation*}
% \qed
\end{proof}

Lemma~\ref{lemma:average} establishes the relationship among $\pi_{C^*}$, $\pi_{A'}$, and $\pi_{B'}$ of a node on $T$.
With this lemma, we obtain an upper bound for each node:

% \begin{proposition}
% A positive \mrb of the plan must be $RB(\pi_{C^*})$ of a node on $T$. 
% \end{proposition}
%\jy{This proposition statement is confusing to me.}

\begin{lemma}\label{lemma:induction}
Given a node $N$ with depth $d$ in the binary tree $T$, let $\pi_{C^*}(N)$, $\pi_{A'}(N)$, and $\pi_{B'}(N)$ be $\pi_{C^*}$, $\pi_{A'}$, and $\pi_{B'}$ of $N$ respectively. 
$RB(\pi_{C^*}(N))$, $RB(\pi_{A'}(N))$, and $RB(\pi_{B'}(N))$ are all upper bounded by 
\begin{equation}
    \dfrac{[1-(\sqrt{\dfrac{2}{3}})^{d+1}]}{1-\sqrt{\dfrac{2}{3}}}20\sqrt{n}
\end{equation}
where $n$ is the number of objects in the instance. 
\end{lemma}

\begin{proof}
The conclusion can be proven by induction.\\

When $d=0$, the dependency graph at the root node $r$ has $n$ start vertices and $n$ goal vertices. 
We have $RB(\pi_{C^*})\leq 20\sqrt{n}, RB(\pi_{B'})=0$. 
According to Lemma \ref{lemma:average}, 
$RB(\pi_{A'})\leq 10 \sqrt{n}$.
The conclusion holds.

Assume that the conclusion holds for all the nodes with depth less than or equal to $k$.
Given an arbitrary node $N$ in the depth $k$, 
let the left and right children of $N$ be $L$ and $R$, which are the nodes with depth $k+1$. 
The corresponding dependency graphs have at most $(2/3)^{k+1}\cdot 2n$ vertices respectively. 
$$RB(\pi_{C^*}(L))\leq RB(\pi_{C^*}(N)) + \sqrt{(2/3)^{k+1}\cdot 2n}\cdot 10\sqrt{2}$$
$$RB(\pi_{C^*}(R))\leq RB(\pi_{A'}(N)) + \sqrt{(2/3)^{k+1}\cdot 2n}\cdot 10\sqrt{2}$$
Since $\pi_{B'}(L)=\pi_{A'}(N)$ and $\pi_{B'}(R)=\pi_{B'}(N)$, upper bound for nodes with depth $k$ holds for $\pi_{B'}(L)$ and $\pi_{B'}(R)$. Therefore, the running buffer size for the depth $k+1$ nodes have an upper bound
\begin{equation}
    \begin{split}
        & \dfrac{[1-(\sqrt{2/3})^{k+1}]}{1-\sqrt{2/3}}20\sqrt{n} + \sqrt{(2/3)^{k+1}\cdot 2n}\cdot 10\sqrt{2}\\ 
        = &\dfrac{[1-(\sqrt{2/3})^{k+2}]}{1-\sqrt{2/3}}20\sqrt{n}        
    \end{split}
\end{equation}

Therefore, the upper bound holds for all the nodes of depth $k+1$. 
With induction, the lemma holds.
% \qed
\end{proof}

With Lemma~\ref{lemma:induction}, we have \mrb is bounded by $\dfrac{20}{1-\sqrt{2/3}}\sqrt{n}$.

\begin{theorem}\label{t:urbm-upper}
For \urbm with $n$ uniform cylinders, a polynomial time algorithm can compute a
plan with $O(\sqrt{n})$ \rb, which implies that \mrb is  bounded by $O(\sqrt{n})$.
\end{theorem}

\section{Fast Algorithms for \lrbm and \urbm}\label{sec:algorithms}
In this section, we first describe a dynamic programming-based method for 
\lrbm (Sec.~\ref{subsec:dp}). 
Then, we propose a priority queue based method in Sec.~\ref{subsec:pqs} for \urbm. 
Finally, a significantly faster 
depth-first modification of DP for computing \mrb is provided in 
Sec.~\ref{subsec:dfsdp}. We mention that, we also developed an integer
programming model, denoted $\ilpmrb$, for computing the minimum total number 
of buffers needed subject to the \mrb constraint, which is compared 
with the algorithm that computes the minimum total buffer without 
the \mrb constraint, denoted as $\ilpfvs$, from \cite{han2018complexity}.
A brief description of $\ilpmrb$ is given in Sec.~\ref{subsec:ilp}.

\vspace{-1mm}
\subsection{Dynamic Programming (DP) for \lrbm}\label{subsec:dp}
%For \lrbm, a rearrangement plan can be represented with an ordering of the objects based on the time moving out of the start poses or the time moving into the goal poses. Due to the high degree of freedom in \lrbm, a class of rearrangement plans may be represented by the same object ordering. We propose strategies to find the optimal plan represented by a certain object ordering. In this way, we transform \lrbm into a permutation problem of objects. After that, the optimal ordering can be found with a dynamic programming(DP) algorithm.
As mentioned in Sec.~\ref{subsec:logvrb}, a rearrangement plan in \lrbm can be represented
as a linear ordering of object labels where objects will be moved out of the start pose
based on the order. That is, given an ordering of objects, $\pi$, we start with $o_{\pi(1)}$. 
If $x^g_{\pi(1)}$ is not occupied, then $o_{\pi(1)}$ is directly moved there. Otherwise, 
it is moved to buffer location. We then continue with the second object in the order, and 
so on. After we work with each object in the given order, we always check whether objects in 
buffer can be moved to their goals, and do so if an opportunity presents. 
We now describe a \emph{dynamic programming} (DP) algorithm for computing an ordering that 
yields the \mrb. 

%The strategies denoted by \StratStart and \StratGoal, are shown below. Given an object ordering $\pi$, \StratStart and \StratGoal interpret $\pi$ as an ordering of objects sorted by the increasing time of leaving the start poses and arriving at the goal poses respectively. Among the class of rearrangement plans represented by the same ordering, the one found by \StratStart follows the rule that objects in the buffer will go to the goal poses as soon as possible. Similarly, the plan found by \StratGoal follows the rule that objects will not go to the buffer until they have to.
%\begin{definition}[\StratStart]
%For each object $o_i$ in the order of $S$
%\begin{enumerate}
%    \item If $c^g_i$ is accessible, move $o_i$ to $c^g_i$. Otherwise, move $o_i$ to the buffer.
%    \item For each object $o_j$ in the buffer, move $o_j$ to $c^g_j$ if $c^g_j$ is accessible after the movement of $o_i$.
%\end{enumerate}
%\end{definition}

%\begin{definition}[\StratGoal]
%For each object $o_i$ in the order of $S$
%\begin{enumerate}
%    \item Move all the objects occupying $c^g_i$ to the buffer.
%    \item Move $o_i$ to $c^g_i$
%\end{enumerate}
%\end{definition}

%\begin{corollary}\label{cor:strategies}
%By enumerating all the object orderings, both \StratStart and \StratGoal can find the optimal %rearrangement plan.
%\end{corollary}
%The proof of Corollary~\ref{cor:strategies} is available in the extended version of this paper on arXiv. 

%\subsubsection{Dynamic Programming Approach}
The pseudo-code of the algorithm is given in Algo.~\ref{alg:dp}. The algorithm maintains 
a search tree $T$, each node of which represents an arrangement where a set of objects
$S$ have left the corresponding start poses. We record the objects currently at the 
buffer ($T[S]$.b) and the minimum running buffer from the start arrangement 
$\mathcal{A}_1$ to the current arrangement ($T[S].\mrb$). The DP starts with an empty 
$T$. We let the root node represent $\mathcal A_1$ (line 1). At this moment, there is
no object in the buffer and the \mrb is 0(line 2-3). And then we enumerate all the arrangements with $|S|=$ 1, 2, $\cdots$ and finally $n$(line 4-5). For arbitrary $S$, 
the objects at the buffer are the objects in $S$ whose goal poses are still occupied by 
other objects(line 6), i.e., $\{o\in S | \exists o' \in \mathcal{O} \backslash S, 
(o,o')\in A\}$, where $A$ is the set of arcs in $\ldg$. $T[S].\mrb$, the minimum running
buffer from the root node $T[\emptyset]$ to $T[S]$, depends on the last object $o_i$ 
added into $S$ and can be computed by enumerating $o_i$ (line 7-20):
\vspace{2mm}
$$
\begin{array}{l}
    T[S].\mrb = \displaystyle\min_{o_i\in S}\max( T[S\backslash \{o_i\}].\mrb, \ |T[S].b|, \\ 
    \qquad\qquad\qquad\qquad\qquad |T[S\backslash \{o_i\}].b|\ +\text{TC}(S\backslash \{o_i\},S)),
\end{array}
\vspace{2mm}
$$
where the \emph{transition cost} TC is given as 
\vspace{1mm}
$$
\text{TC}(S\backslash \{o_i\},S)=\begin{cases}1,\  o_i \in T[S].b,\\ 
0,\  otherwise,\end{cases}
\vspace{1mm}
$$
with $x^g_i$ currently occupied (line 10), the transition cost is due to objects 
dependent on $o_i$ cannot be moved out of the buffer before moving $o_i$ to the 
buffer(line 11). If $T[S].\mrb$ is minimized with $o_i$ being the last object in 
$S$ from the starts, then $T[S\backslash \{o_i\}]$ is the parent node of $T[S]$ 
in $T$ (line 14-16). Once $T[\mathcal O]$ is added into $T$, $T[\mathcal{O}].\mrb$ 
is the \mrb of the instance (line 17) and the path in $T$ from $T[\emptyset]$ to
$T[\mathcal{O}]$ is the corresponding solution to the instance.

\begin{algorithm}
\begin{small}
    \SetKwInOut{Input}{Input}
    \SetKwInOut{Output}{Output}
    \SetKwComment{Comment}{\% }{}
    \caption{ Dynamic Programming}
		\label{alg:dp}
    \SetAlgoLined
		\vspace{0.5mm}
    \Input{$G_{\mathcal{A}_1, \mathcal{A}_2}(\mathcal{O},A)$: labeled dependency graph}
    \Output{$\mrb$: the minimum number of running buffers}
		\vspace{0.5mm}
		$T.root \leftarrow \emptyset$ \\
		\vspace{0.5mm}
		$T[\emptyset].b \leftarrow\emptyset$ \Comment{{\footnotesize objects currently at the buffer}}
		\vspace{0.5mm}
		$T[\emptyset].\mrb \leftarrow 0$ \Comment{{\footnotesize current minimum running buffer}}
		\vspace{0.5mm}
		\For{$1\leq k\leq |\mathcal{O}|$}{
		\Comment{\footnotesize{enumerate cases where $k$ objects have left the start poses}}
		\For{$S\in \text{k-combinations of } \mathcal{O}$}{
		$T[S].b\leftarrow \{o\in S \mid \exists o' \in \mathcal{O} \backslash S, (o,o')\in A\}$\\
		\Comment{{\footnotesize Find the \mrb from $T[\emptyset]$ to $T[S]$}}
		$T[S].\mrb  \leftarrow \infty$\\
		\For{$o_i \in S$}{
		$parent = S\backslash \{o_i\}$\\
		\If{$o_i\in T[S].b$}{
		$RB \leftarrow \max(T[parent].\mrb, |T[S].b|$, $|T[parent].b|+1$)
		}
		\Else{
		$RB \leftarrow \max(T[parent].\mrb, |T[S].b|$)
		}
		\If{RB$<T[S].\mrb$}{
		$T[S].\mrb \leftarrow RB$\\
		$T[S].parent \leftarrow$ parent
		}
		}
		}
		}
		\vspace{0.5mm}
		\Return $T[\mathcal{O}]$.\mrb\\
\end{small}
\end{algorithm}

For the \lrbm instance in Fig.~\ref{fig:ex-prob}, Tab.~\ref{Tab:DP} shows $T[S].\mrb$ 
with different last-object options when $S=\{o_2, o_5, o_6\}$. If the last object $o_i$ is 
$o_5$, then we need to move $o_5$ into the buffer before moving $o_6$ out of the buffer. Therefore, 
even though buffer sizes of the parent node and the current node are both 2, there is a 
moment when all of the three objects are at the buffer. However, when we choose $o_2$ or 
$o_6$ as the last object to add, the $T[S].\mrb$ becomes 2.
\begin{center}
\begin{table}[h!]
    \caption{\label{Tab:DP}$T[S].\mrb$ for different last objects ($[p]= [parent]$)}
    \begin{tabular}{p{4.6em}|p{4.6em}|p{4.6em}|p{4.6em}|p{4.6em}}
        \hline
        Last object & $T[p].\mrb$ & $T[p].b$ & $T[S].b$ & $T[S].\mrb$\\
        \hline
        $o_2$ & 1 & \{$o_5$\} & \{$o_2$, $o_5$\} & 2\\
        $o_5$ & 2 & \{$o_2$, $o_6$\} & \{$o_2$, $o_5$\} & 3\\
        $o_6$ & 2 & \{$o_2$, $o_5$\} & \{$o_2$, $o_5$\} & 2\\
    \end{tabular}
\end{table}
\vspace{-3mm}
\end{center}

\subsection{A Priority Queue based method for \urbm}\label{subsec:pqs}
Similar to \lrbm, rearrangement plans in \urbm can be represented by a linear 
ordering of goal vertices in $\udg$. We can compute the ordering that yields 
\mrb by maintaining a search tree like in Algo.~\ref{alg:dp}. Each node $T[V]$ 
in the tree represents an arrangement where a set of goal vertices $V$ have 
been filled. The remaining dependencies of $T[V]$ is an induced graph of $\udg$, 
formed from $V(\udg)\backslash (V\cup N(V))$ where $N(V)$ is the neighbors of 
$V$ in $\udg$. The running buffer size of $T[V]$ is $|N(V)|-|V|$. Given an 
induced graph $I(V)$, denote the goal vertices with no more than one neighbor 
in $I(V)$ as \emph{free goals}. We make two observations. First, given an 
induced graph $I(V)$, we can always prioritize the free goals in terms of 
the order to fill without optimality loss.
Second, multiple free goals may be generated as a goal vertex is filled. For example, 
in the instance shown in Fig.~\ref{fig:ex-prob}(a), when the goal representing 
$c^5_g$ is filled, $c^2_g$, $c^3_g$, and $c^4_g$ become free goals and can be added to 
the linear ordering in an arbitrary order. In conclusion, the necessary nodes 
(nodes without free goals in the induced graph) in the search tree are sparse 
and enumerating nodes with DP carries much overhead. 

As such, instead of exploring the search tree layer by layer like DP, we maintain 
a sparse tree with a priority queue $Q$. While each node still represents an 
arrangement, each edge in the tree represents either an action moving an object 
to the buffer or multiple actions filling free goal poses. We always 
pop out and develop the node with the smallest \mrb in $Q$. If a child node of 
the one that we develop already exist in the tree but is with smaller \mrb than 
previously claimed, we will update the parent of the child node into the node we 
are developing. The \mrb of the node representing $\mathcal A_2$ sets an upper bound 
of the solution and nodes in $Q$ with larger \mrb will be pruned away. The algorithm 
terminates when $Q$ is empty.
We denote this priority queue-based search method \PQS.

\vspace{-1mm}
\subsection{Dynamic Programming with Depth-First Exploration}\label{subsec:dfsdp}
Both \lrbm and \urbm can be viewed as solving a series of decision problems, i.e., 
asking whether we can find a rearrangement plan with $k= 1, 2, \ldots$ running buffers. 
As dynamic programming is applied to solve such decision problems, instead of 
performing the more standard breadth first exploration of the search tree, we 
identified that a depth-first exploration is much more effective. 
We call this variation of dynamic programming, which is a fairly 
straightforward alteration of a standard DP procedure. 
%\DFSDP \cite{wang2021uniform}.
%
Essentially, \DFSDP fixes a $k$ and checks whether there is a plan requiring 
no more than $k$ running buffers. As the search tree (see Sec.~\ref{subsec:dp}) 
is explored, depth-first exploration is used instead of breadth-first. The 
intuition is that, when there are many rearrangement plans on the search tree 
that do not use more than $k$ running buffers, depth-first search will quickly 
find such a solution, whereas a standard DP must grow the full search tree 
before returning a feasible solution. 
A similar depth-first exploration heuristic is used in \cite{wang2021uniform}. 
%With \DFSDP, we still develop a search tree as described in Sec.~\ref{subsec:dp}, 
%each node of which represents an arrangement with a certain set of objects moved 
%away from the start poses. When a search node is added to the tree, we only need 
%to check in linear time whether the transition from the parent arrangement to the 
%child arrangement can be done with $k$ running buffers. We develop the tree in a 
%depth first manner until either the node representing $\mathcal A_2$ is added into 
%the tree or all the leaf nodes in the tree are proven to be dead ends. T

\subsection{Minimizing Total Buffers Subject to \mrb Constraints}\label{subsec:ilp}
Let binary variables $c_{i,j}$ represent $\ldg$: $c_{i,j}=1$ if and only if $(i,j)$ is in the arc set of $\ldg$. Let $y_{i,j}(1\leq i<j\leq n)$ be the binary 
sequential variables: $y_{i,j}=1$ if and only if $o_i$ moves out of the start pose before $o_j$. 
% \mrb can be expressed based on three constraints: 
% First, \mrb is at least the size of the running buffer at any moment; 
% Second, an object $o_j\in \mathcal{O}$ is at the goal pose if and only 
% if all the objects $o_k\in \mathcal{O}$ with $c_{j,k}=1$ have left the start 
% poses. Third, an object $o_j\in \mathcal{O}$ is at the buffer if and only 
% if $o_j$ is neither at the start pose nor at the goal pose.
%
We further introduce two sets of binary variables $g_{i,j}$ and $b_{i,j}(1\leq i,j\leq n)$ to indicate object positions at each moment. 
$g_{i,j}=1$ indicates that $o_j$ has no dependency on other objects when moving $o_i$ from the start pose. In other words, the goal pose of $o_j$ is available at the moment. $b_{i,j}=1$ indicates that $o_j$ stays at the buffer after moving $o_i$ away from the start pose.
Finally, binary variables $B_i=1$ if and only if $o_i$ is moved to a buffer at some point.
The objective function consists of two terms: 
the total buffer term and running buffer term.
The total buffer term, scaled by $\alpha$, counts the number of objects that need buffer locations.
The running buffer is represented with an integer variable $K$ and scaled by $\beta$.
To minimize total buffers subject to \mrb constraints,
we set $\alpha=1,\beta=n$.
The objective function is adaptive for different demands on rearrangement plans. 
Specifically, when $\alpha=0$,$\beta > 0$,
the MIP model minimizes \mrb.
When $\alpha > 0$, $\beta =0$, the MIP model minimizes total buffers, 
i.e. total actions in the rearrangement plan.
When $\alpha/\beta >n-1$, the MIP model first minimizes total buffers, 
and then minimizes running buffers.
When $\beta/\alpha >n-1$, the MIP model first minimizes running buffers, 
and then minimizes total buffers.

In the MIP model, Constraints~\ref{eq:c1} imply the rules for sequential variables.
Constraints~\ref{eq:c2} imply that $B_j=1$ if $o_j$ has been at buffers in the plan.
Constraints~\ref{eq:c3} imply that the running buffer $K$ is lower bounded by the maximum number of objects concurrently placed in buffers.
With Constraints~\ref{eq:c4} and \ref{eq:c5}, $g_{i,j}=0$ if and only if $o_j$ depends on an object $o_k$ which is still at the start pose when $o_i$ is moved.
With Constraints~\ref{eq:c6}-\ref{eq:c8}, 
$b_{i,j}=1$ if and only if $o_j$ is moved before $o_i$ and the goal pose is still unavailable when $o_i$ is moved from the start pose.
%\jy{``constraint x'' $\to$ ``Constraint x''}

\begin{equation}
    \arg \min \alpha [\sum_{i=1}^{n}B_i]+\beta K
\end{equation}
\begin{equation}\label{eq:c1}
    0\leq y_{i,j}+y_{j,k}-y_{i,k}\leq 1 \ \ \  \forall 1\leq i < j < k \leq n
\end{equation}
% \begin{equation}\label{eq:c2}
%     y_{i,i}=1 \ \ \  \forall 1 \leq i \leq n
% \end{equation}
\begin{equation}\label{eq:c2}
    B_j \geq \sum_{1\leq i\leq n} \dfrac{b_{i,j}}{n}\ \ \  \forall 1\leq j\leq n
\end{equation}
\begin{equation}\label{eq:c3}
    K \geq \sum_{1\leq j\leq n} b_{i,j}\ \ \  \forall 1\leq i\leq n
\end{equation}
\begin{equation}\label{eq:c4}
    \begin{split}
    \sum_{1\leq k < i} \dfrac{c_{j,k}(1-y_{k,i})}{n} + \sum_{i < k \leq n} \dfrac{c_{j,k}y_{i,k}}{n} \leq 1-g_{i,j} \\
    \forall 1\leq i,j \leq n
    \end{split}
\end{equation}
\begin{equation}\label{eq:c5}
    \begin{split}
    1-g_{i,j}\leq \sum_{1\leq k < i} c_{j,k}(1-y_{k,i}) + \sum_{i < k \leq n} c_{j,k}y_{i,k}\\ 
    \forall 1\leq i,j \leq n
    \end{split}
\end{equation}
\begin{equation}\label{eq:c6}
\begin{split}
    \dfrac{g_{i,j}+y_{i,j}}{2}\leq 1-b_{i,j} \leq g_{i,j}+y_{i,j}\\
    \forall 1\leq i<j \leq n
\end{split}
\end{equation}
\begin{equation}\label{eq:c7}
\begin{split}
    \dfrac{g_{j,i}+(1-y_{i,j})}{2}\leq 1-b_{j,i} \leq g_{j,i}+(1-y_{i,j})\\
    \forall 1\leq i<j \leq n
\end{split}
\end{equation}
\begin{equation}\label{eq:c8}
b_{i,i}=1-g_{i,i}
\end{equation}

\section{Experimental Studies}\label{sec:evaluation}
Our evaluation focuses on uniform cylinders, given their prevalence in 
practical applications. 
For simulation studies, instances with different object densities are created, as measured 
by \emph{density level} $\rho := n\pi r^2/(h*w)$, where $n$ is the number of objects and 
$r$ is the base radius. $h$ and $w$ are the height and width of the workspace.
%\jy{Kai: please update $D$ to $\rho$, which is a more standard symbol for density.}
In other words, $\rho$ is the proportion of the tabletop surface occupied by objects. 
%We notice that instances with a fixed $D$ have roughly the same number of dependencies on average for each object, regardless of the number of objects in the environment. 
%For example, when $D=0.3$, each object has averaging 1.25 dependencies. 

The evaluation is conducted on both random object placements and manually 
constructed difficult setups (e.g., dependency grids with $MRB = \Omega(\sqrt{n}$)).
For generating test cases with high $\rho$ value, we invented a physic engine 
(we used Gazebo) based approach for doing so. Within a rectangular box, we sample 
placements of cylinders at lower density and then also sample locations for some 
smaller ``filler'' objects (see Fig.~\ref{fig:compression}, left). From here, 
one side of the box is pushed to reach a high density setting 
(Fig.~\ref{fig:compression}, right), which is very difficult to generate via random 
sampling. By controlling the ratio of the two types of objects, different density 
levels can be readily obtained. Fig.~\ref{fig:density} shows three random object placements 
for $\rho = 0.2, 0.4$ and $0.6$.

\begin{figure}[h!]
    \vspace{1mm}
    \centering
    \includegraphics[width=0.8\columnwidth]{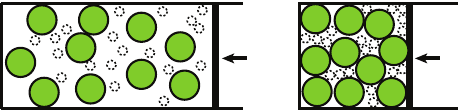}
    \vspace{1mm}
    \caption{Generating dense instances using a physics-engine based simulator
    through     compression of the left scene to the right scene.}
    \label{fig:compression}
\end{figure}

\begin{figure}[h!]
    \vspace{-1mm}
    \centering
    \includegraphics[width=0.75\columnwidth]{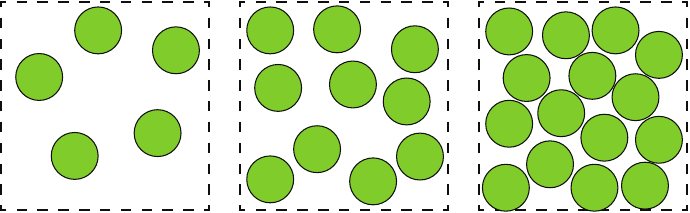}
    \vspace{1mm}
    \caption{Unlabeled arrangements with $\rho=0.2, 0.4, 0.6$ respectively.}
    \label{fig:density}
    \vspace{-1mm}
\end{figure}

From two randomly generated object placements with same $\rho$ and $n$ values, a 
\urbm instance can be readily created by superimposing one over the other. \lrbm
instances can be generated from \urbm instances by assigning each object a random 
label in $[n]$ for both start and goal configurations.

%In random scenario, we consider cylindrical objects, whose bottoms are unit discs, placed in a square environment without collisions. 
%The size of the square environment are determined by the density level $D$ and the number of objects $n$. 
%Unlabeled arrangements are generated under uniform distribution of objects: Object centroids are assigned one by one following the uniform distribution in the workspace.
%An object is placed if the assigned location is out of collision with previous-placed objects. Otherwise, other randomly generated locations will be assigned to it until a collision-free location is found. 
%For each instance, $\mathcal A_1$ and $\mathcal A_2$ are randomly chosen from the generated unlabeled arrangements and in \lrbm, the labels of the each object are also randomly assigned.
%
%When the desired density level is too high, the generation process becomes slow: collision-free positions are hard to find for objects at the bottom of the list. In this case, we generate the arrangements by compressing sparse instances(Fig.~\ref{fig:conpression}) in Gazebo, a simulation software\cite{koenig2004design}. To enrich the randomness of the object positions, smaller cylindrical obstacles are temporarily added into the environment so that the density is balanced all over the environment.
%
%In the special scenario, we evaluate the special cases discussed in the previous sections whose $\mrb=\Theta(\sqrt{n})$.

The proposed algorithms are implemented in Python and all experiments are executed 
on an Intel$^\circledR$ Xeon$^\circledR$ CPU at 3.00GHz. For solving ILP, Gurobi 9.16.0 \cite{gurobi} is used.

\vspace{-1mm}
\subsection{Labeled Rearrangement over Random Instances }
In Fig.~\ref{fig:LabeledAlgorithms}, we compare the effectiveness of the DP and \DFSDP, in terms of computation time and success rate, for different densities. 
Each data point is the average of $30$ test cases minus the unfinished ones, if any, 
subject to a time limit of $300$ seconds per test case. For \lrbm, we are able to 
push to $\rho = 0.4$, which is fairly dense. 
The results clearly demonstrate that \DFSDP significantly outperform the baseline
DP.
Based on the evaluation, both methods can be used to tackle practical sized problems 
(e.g., tens of objects), with \DFSDP demonstrating superior efficiency and robustness. 
\begin{figure}[h!]
\vspace{1.5mm}
    \centering
    \begin{overpic}[width=1\columnwidth]{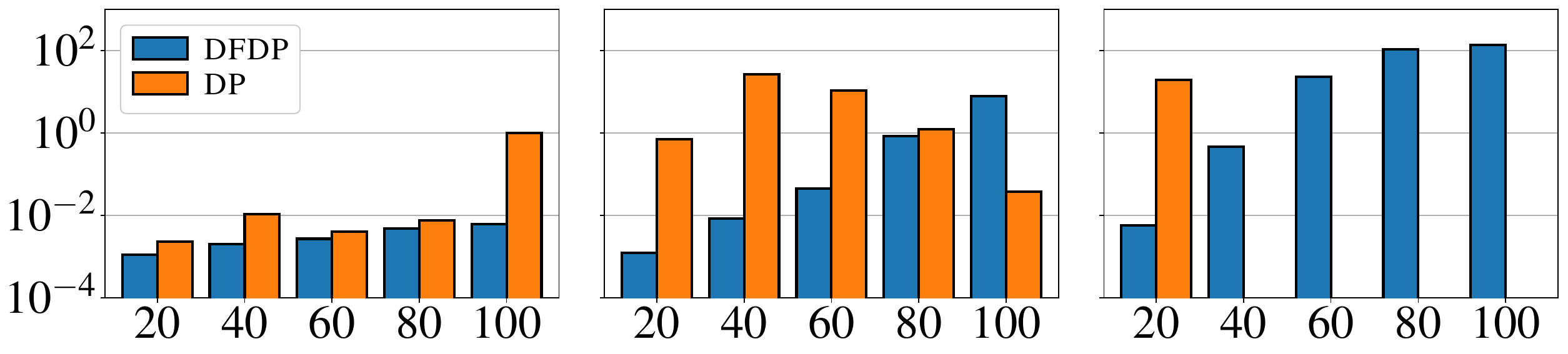}
    \put(0.5,-22.5){ \includegraphics[width=0.982\columnwidth]{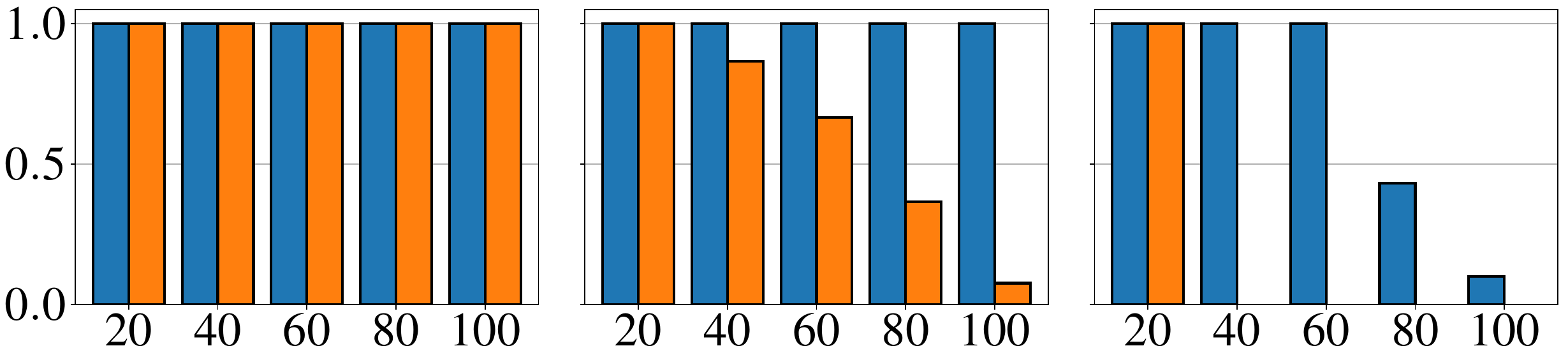}}
    \end{overpic}
    \vspace{16mm}
    \caption{Performance of \DFSDP and DP over \lrbm. The top row shows the average
    computation time (s) and the bottom row the success rate, for density levels
    $\rho=0.2, 0.3$, $0.4$, from left to right. The $x$-axis denotes the number of 
    objects involved in a test case.}
    \label{fig:LabeledAlgorithms}
\end{figure}

The actual \mrb sizes for the same test cases from Fig.~\ref{fig:LabeledAlgorithms} 
are shown in Fig.~\ref{fig:LabeledResults}
on the left. We observe that \mrb is rarely very large even for fairly large \lrbm
instances. The size of \mrb appears correlated to the size of the largest connected
component of the underlying dependency graph, shown in Fig.~\ref{fig:LabeledResults}
on the right.
\begin{figure}[h!]
\vspace{1mm}
    \centering
\begin{overpic}
[width=0.23\textwidth]{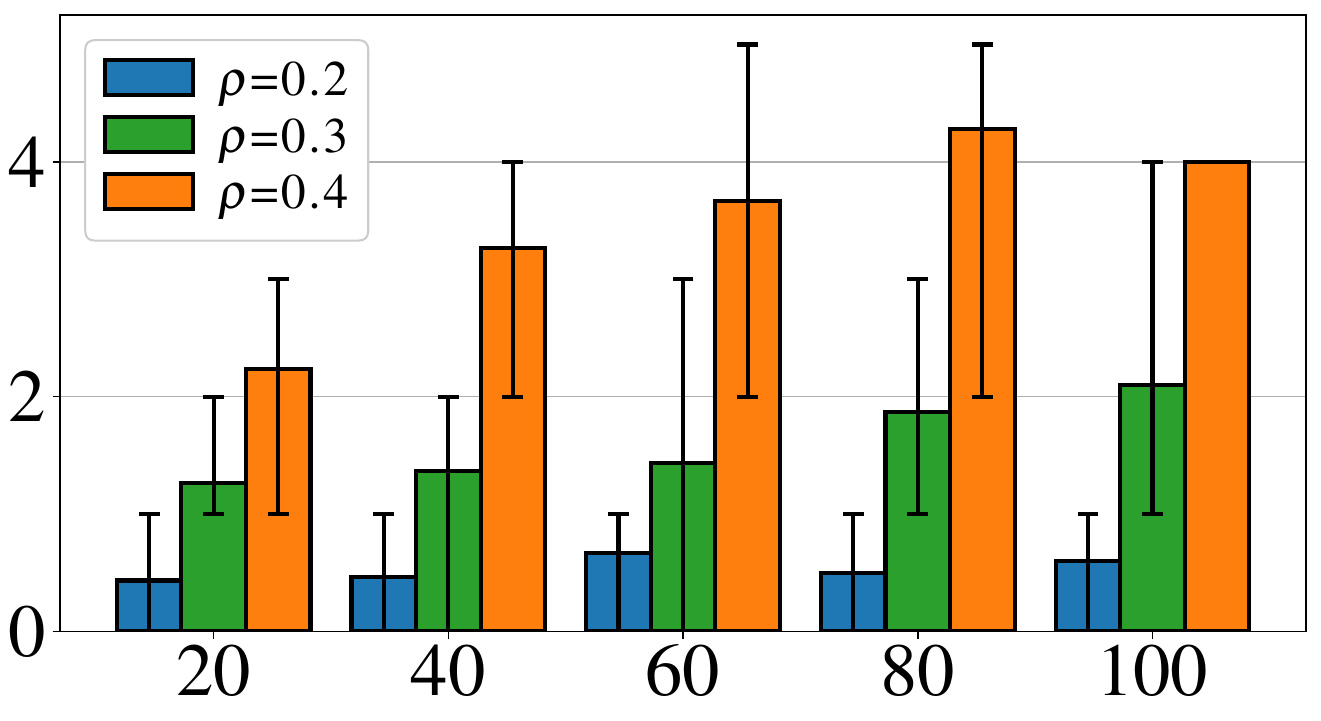}
\end{overpic}
\begin{overpic}
[width=0.23\textwidth]{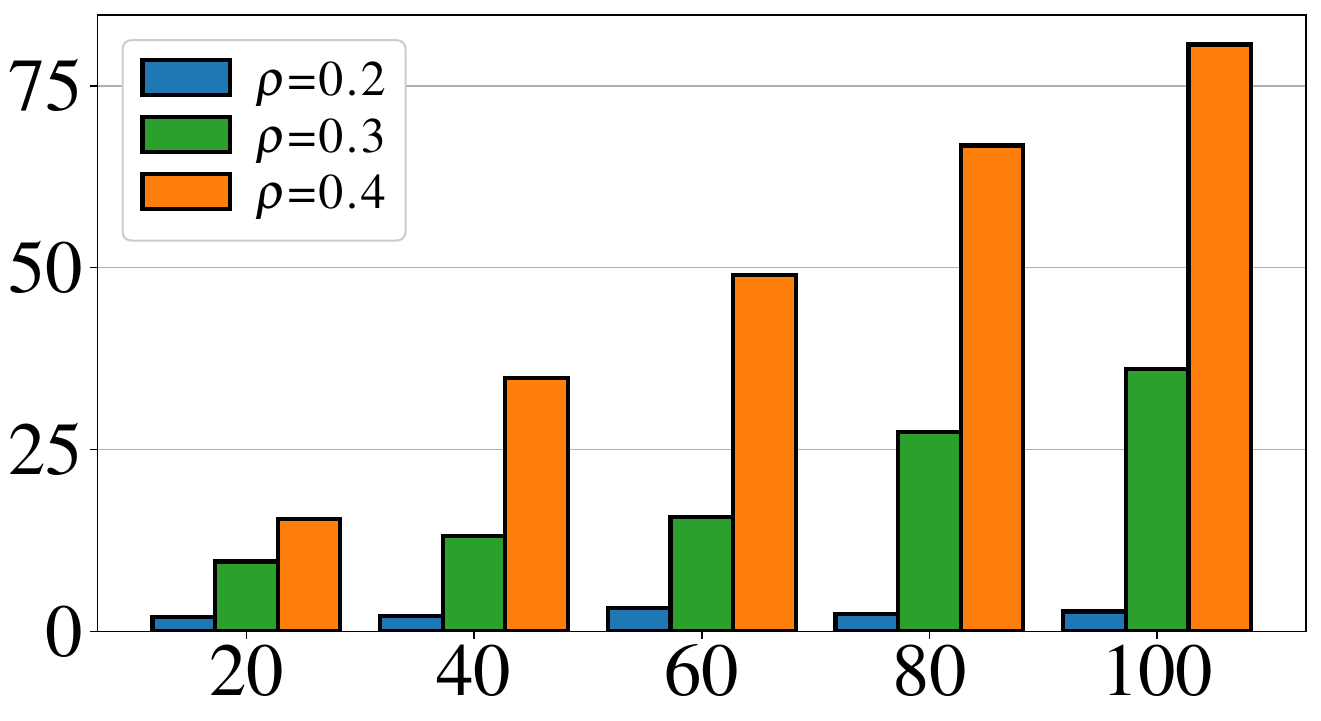}
\end{overpic}
    \caption{For \lrbm instances with $\rho = 0.2$-$0.4$ and $n=20$-$100$, the 
    left figure shows average \mrb size and range. The right figure shows
    the size of the largest connected component of the dependency graph.}
    \label{fig:LabeledResults}
        \vspace{-1mm}
\end{figure}
% \jy{We should probably get rid of the range for the second 
% figure in Fig.~\ref{fig:LabeledResults}; otherwise we probably 
% need to add to all other figures...}

For $\lrbm$ with $\rho = 0.3$ and $n$ up to $50$, we computed the \fvs sizes 
using $\ilpfvs$ (which does not scale to higher $\rho$ and $n$) and compared that 
with the \mrb sizes, as shown in Fig.~\ref{fig:MRBFVS} (a). We observe that the
\fvs is about twice as large as \mrb, suggesting that \mrb provides more reliable 
information for estimating the design parameters of pick-n-place systems. For 
these instances, we also computed the total number of buffers needed subject to 
the \mrb constraint using $\ilptb$. Out of about $150$ instances, only $1$ showed 
a difference as compared with \fvs (therefore, this information is not shown in 
the figure). In Fig.~\ref{fig:MRBFVS} (b), we provided computation time 
comparison between $\ilpfvs$ and $\ilptb$, showing that $\ilptb$ is practical,
if it is desirable to minimize the total buffers after guaranteeing the 
minimum number of running buffers. 

\begin{figure}[h!]
\vspace{1mm}
    \centering
    \begin{overpic}[width=0.33\textwidth]{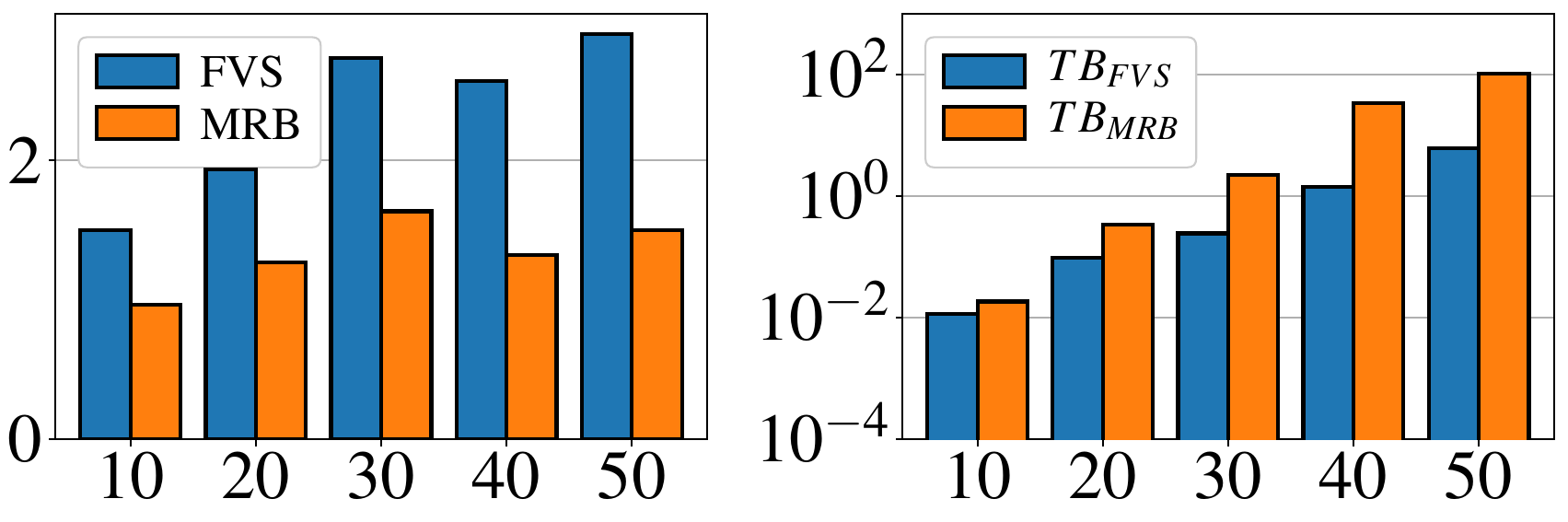}
    \put(22,-3.5){{\small (a)}}
    \put(76,-3.5){{\small (b)}}
    \end{overpic}
    \vspace{2mm}
    \caption{(a) Comparison between size of \mrb and \fvs. (b)  Computation time 
    comparison between $\ilpfvs$ and $\ilptb$. }
    \label{fig:MRBFVS}
    \vspace{-1mm}
\end{figure}

Considering our theoretical findings and the evaluation results, an important 
conclusion can be drawn here is that \mrb is effectively a small constant for
random instances, even when the instances are very large. Also, 
minimizing the total number of buffers used subject to \mrb constraint can 
be done quickly for practical sized problems. 

\subsection{Unlabeled Rearrangement over Random Instances}
For \urbm, we carry out similar performance evaluation as we have done for \lrbm. 
Here, \PQS and \DFSDP are compared. For each combination of $\rho$ and $n$, $100$ 
random test cases are evaluated. Notably, we can reach $\rho = 0.6$ with relative
ease. From Fig.~\ref{fig:UnlabeledAlgorithms}, we observe that \DFSDP is more 
efficient than \PQS, especially for large-scale dense settings. In terms of the \mrb 
size, all instances tested has an average \mrb size between $0$ and $0.7$, which is 
fairly small (Fig.~\ref{fig:UnlabeledResults}). Interestingly, we witness a decrease of 
\mrb as the number of objects increases, which could be due to the lessening 
``border effect'' of the larger instances. That is, for instances with fewer 
objects, the bounding square puts more restriction on the placement of the 
objects inside. For larger instances, such restricting effects become smaller.
We mention that the total number of buffers for random \urbm cases subject to 
\mrb constraints are generally very small. 
%The right figure in Fig.~\ref{fig:UnlabeledResults} shows the 
%number of running buffers we can get "for free" before we do the first branching.

\begin{figure}[h!]
    \vspace{2mm}
    \centering
    \begin{overpic}[width=1\columnwidth]{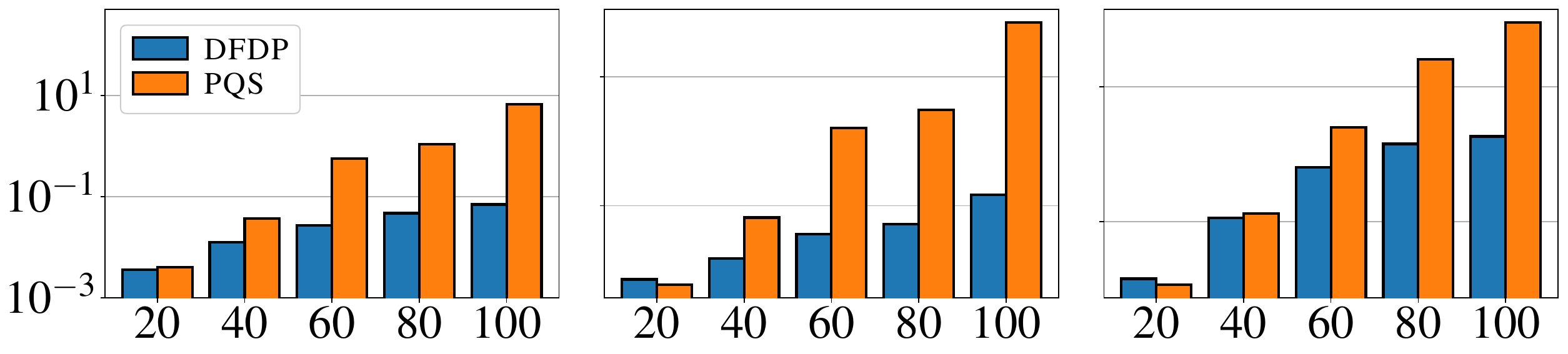}
    \put(0.5,-22.5){ \includegraphics[width=0.982\columnwidth]{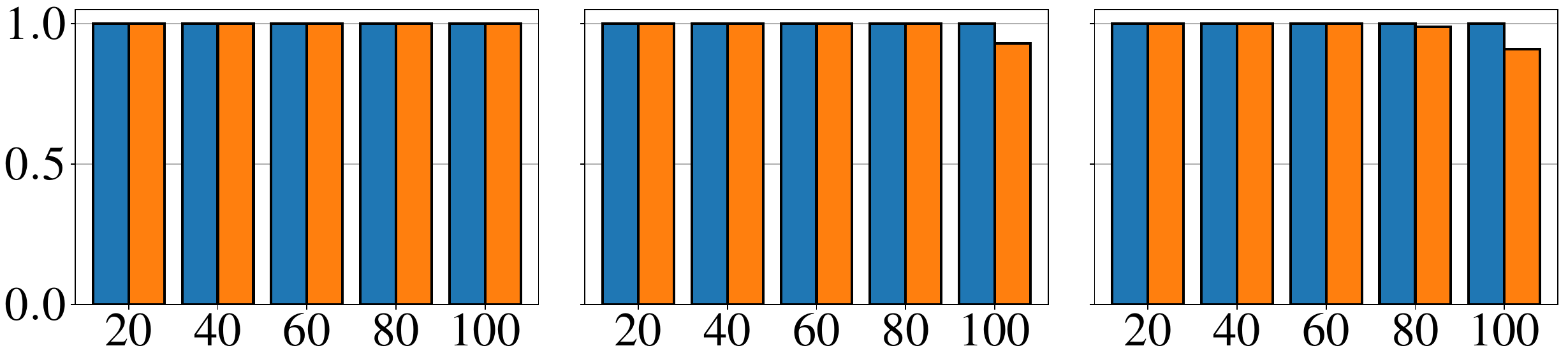}}
    \end{overpic}
        \vspace{16mm}
\caption{Performance of \DFSDP and \PQS over \urbm. The top row shows the average
    computation time and the bottom row shows the success rate, for density levels
    $\rho=0.4, 0.5, 0.6$, from left to right.}
    \label{fig:UnlabeledAlgorithms}
\end{figure}

%\begin{comment}
\begin{figure}[h!]
    \centering
    \vspace{-1mm}
\includegraphics[width=0.27\textwidth]{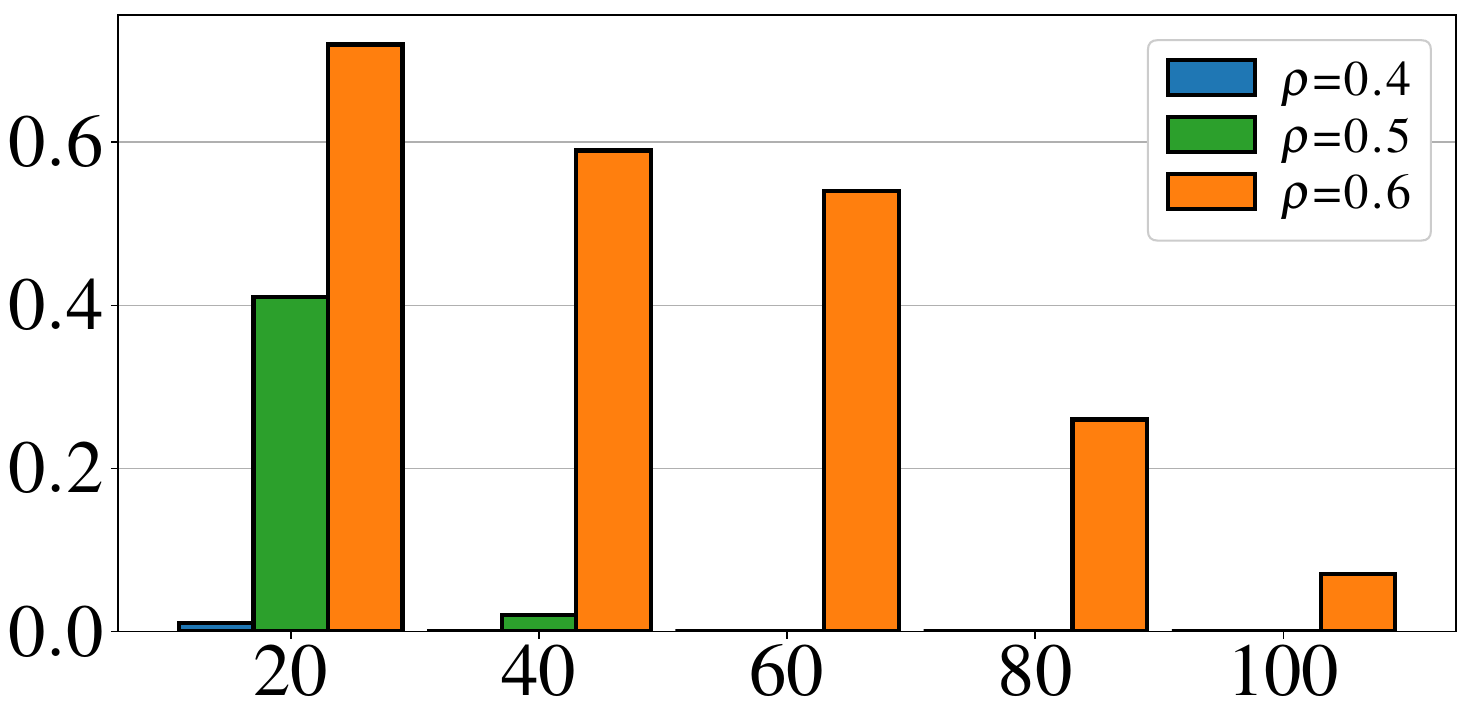}
    \caption{Average \mrb size for \urbm instances with $\rho=0.4-0.6$ and $n=20-100$.
    For $\rho = 0.4$ and $0.5$, the $\mrb$ sizes are near zero as the number of 
    objects goes beyond $20$.}
    \label{fig:UnlabeledResults}
    \vspace{-1mm}
\end{figure}
%\jy{Make Fig.~\ref{fig:UnlabeledResults} shorter, similar to Fig. 14.}
%\end{comment}

\subsection{Manually Constructed Difficult Cases}
In the random scenario, the running buffer size is limited. In particular, for 
\lrbm, the dependency graph tends to consist of multiple strongly connected 
components that can be dealt with independently. We further show the performance 
of \DFSDP on the instances with $\mrb=\Theta(\sqrt{n})$. We evaluate three kinds 
of instances: (1) \textsc{UG}: $m^2$-object \urbm instances whose $\udg$ are 
dependency grid $\mathcal D(m,2m)$ (e.g., Fig.~\ref{fig:DependencyGrid}); (2) 
\textsc{LG}: $m^2$-object \lrbm instances whose start and goal arrangements are 
the same as the instances in (1). (3) \textsc{LC}: $m^2$-object \lrbm instances 
with objects placed on a cycle (Fig.~\ref{fig:lrbm-cycle}). The computation time and 
the corresponding \mrb are shown in Fig.~\ref{fig:SpecialResults}. For 
\textsc{LG} instances, the labels are randomly assigned. We try 30 test cases and 
then plot out the average. We observe that the \mrb are much larger for these 
handcrafted instances as compared with random instances with similar density and
number of objects.

\begin{figure}[h!]
    \vspace{1mm}
    \centering
    \begin{overpic}[width=0.42\textwidth]{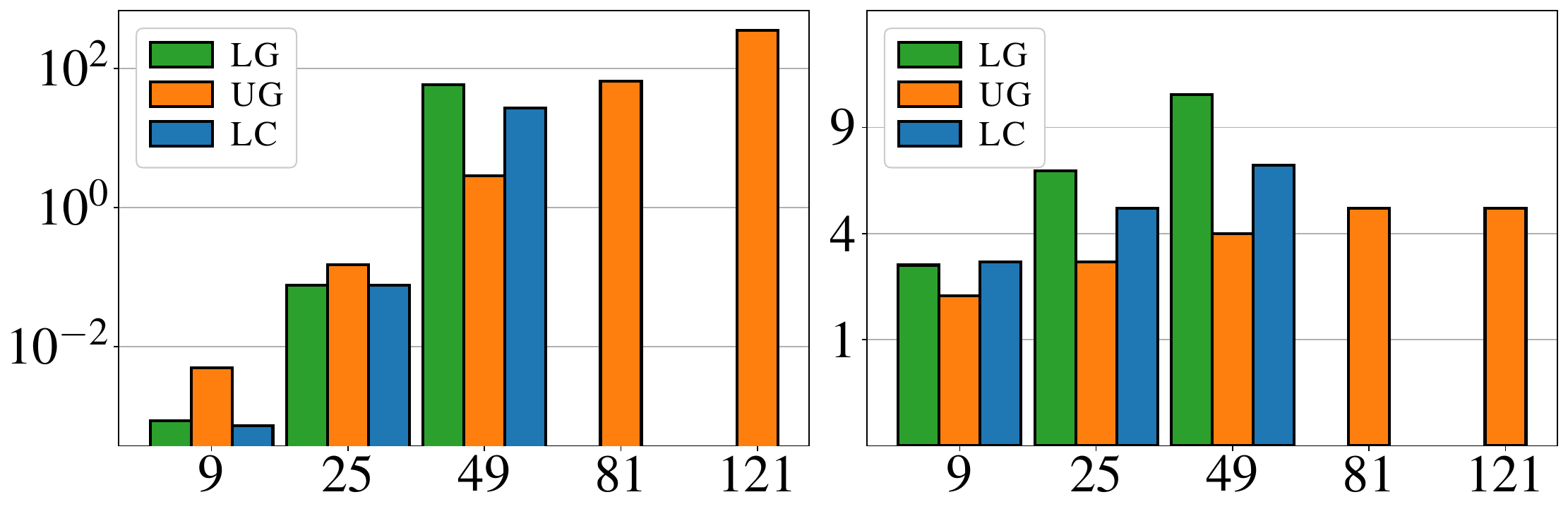}
    \end{overpic}
    %\vspace{-2mm}
    \caption{ For handcrafted cases and different number of objects, the left 
    figure shows the computation time by \DFSDP and the right figure the resulting
    \mrb size.}
    \label{fig:SpecialResults}
    \vspace{-3mm}
\end{figure}

\section{Conclusion and Discussion}\label{sec:conclusion}
In this work, we investigate the problem of minimizing the number of running 
buffers (\mrb) for solving labeled and unlabeled tabletop rearrangement problems 
with overhand grasps (\toro), which translates to finding a best linear ordering 
of vertices of the associated underlying dependency graph. 
For \toro, \mrb is an important quantity to understand as it determines the 
problem's feasibility if only external buffers are to be used, which is 
the case in some real-world applications \cite{han2018complexity}.
Despite the provably high computational complexity that is involved, we provide 
effective dynamic programming-based algorithms capable of quickly computing \mrb 
for large and dense labeled/unlabeled \toro instances.
In addition, we also provide methods for minimizing the total number of buffers 
subject to \mrb constraints.
Whereas we prove that \mrb can grow unbounded for both labeled and unlabeled 
settings for special cases for uniform cylinders, empirical evaluations 
suggest that real-world random \toro instances are likely to have much smaller
\mrb values. 

We conclude by leaving the readers with some interesting open problems. On the 
structural side, while \lrbm in general is proven to be NP-Hard, the computational 
intractability of either \lrbm with uniform cylinders or \urbm in general remains
unresolved. 
As for bounds, the lower and upper bounds of \mrb for \lrbm for uniform cylinders
do not yet agree; can the bound gap be narrowed further?

%\textcolor{red}{In terms of random instances, as the environment density $\rho$ grows, 
%previous works have seen a phase transition in \vsp on random geometric graphs\cite{diaz2000convergence}\cite{penrose2000vertex}. 
%These results motivate us to seek for the phase transition in \mrb with random uniform disk arrangements as well.}

% \newpage
%{\small
\bibliographystyle{IEEEtran}
% \bibliography{..//bib/bib}
\bibliography{bib}
%}

%\section{Supplementary Material}\label{sec:proofs}
%\input{texs/20-app}

\end{document}